\documentclass[journal]{IEEEtran}
\usepackage{amsmath,amsfonts}
\usepackage{algorithm}
\usepackage{array}
\usepackage[caption=false,font=normalsize,labelfont=sf,textfont=sf]{subfig}
\usepackage{textcomp}
\usepackage{stfloats}
\usepackage{url}
\usepackage{verbatim}
\usepackage{graphicx}
\usepackage{bm}
\usepackage{cite}
\usepackage{amsthm}

\newtheorem{theorem}{Theorem}
\usepackage{booktabs}
\usepackage{tabularx}
\usepackage{algpseudocode}
\usepackage{longtable}
\usepackage{placeins}
\usepackage[table]{xcolor}
\usepackage{multirow}
\usepackage{hyperref}
\hypersetup{
    hidelinks
}
\newtheorem{proposition}{Proposition}

\newtheorem{lemma}{Lemma}

\definecolor{lightblue}{RGB}{220, 235, 252}
\definecolor{lightyellow}{RGB}{255, 253, 220}
\definecolor{lightgreen}{RGB}{244, 255, 242}

\def\x{\mathbf{x}}
\def\z{\mathbf{z}}
\def\y{\mathbf{y}}

\def\w{\mathbf{w}}
\def\a{\mathbf{a}}
\def\k{\mathbf{k}}

\def\W{\mathbf{W}}

\newcommand{\T}{\mathbb{T}}
\newcommand{\Z}{\mathbb{Z}}
\newcommand{\R}{\mathbb{R}}
\newcommand{\norm}[1]{\left\| #1 \right\|}

\usepackage{tikz}
\usetikzlibrary{arrows.meta,calc,positioning,backgrounds,decorations.pathreplacing}

\begin{document}

\title{
Dynamics as Code: On Model Compression via Dynamic System
}

\author{
Fan Gao, Wei Su*, Juntong Fan*, Renfeng Peng, Hongyu Liu, Jinqiao Duan, Feng-Lei Fan
\thanks{
Fan Gao and Wei Su are with the School of Information Science and Engineering, Lanzhou University, Lanzhou, Gansu Province, China;
Feng-Lei Fan, Renfeng Peng, and Juntong Fan are with the Department of Data Science, City University of Hong Kong, Hong Kong, China SAR; Hongyu Liu is with the Department of Mathematics, City University of Hong Kong, Hong Kong, China SAR; Jinqiao Duan is with the School of Science, Great Bay University, Dongguan, China; Wei Su (suwei@lzu.edu.cn) and JunTong Fan are co-corresponding authors.
}
}
   
\maketitle
\begin{center}
\footnotesize
This work has been submitted to the IEEE for possible publication.
Copyright may be transferred without notice, after which this version
may no longer be accessible.
\end{center}

\begin{abstract}
The escalating size of pretrained neural networks has rendered model compression a prerequisite for deployment under stringent memory and compute constraints. With the irrational winding as an example, earlier work introduced a dynamic system (DS) paradigm that reconceptualizes compression as compact weight representation: high-dimensional parameters are encoded by the index of a trajectory produced by a dynamic system, from which the vector is recovered during decompression. This mechanism is fundamentally distinct from pruning, quantization, knowledge distillation, and low-rank decomposition. Along this direction, we prove that under a Diophantine condition, a finite trajectory of \(M = O(\epsilon^{-(d+\nu)})\) states in the irrational winding constitutes an \(\epsilon\)-net over the \(d\)-dimensional weight space, thereby linking state resolution, decompression error, and compression ratio in a predictable manner. Furthermore, we propose a generalized DS-based model compression framework by unifying four DS families---space-filling curves (Hilbert, Peano, Morton/Z-order, Snake), chaotic systems (Lorenz), congruential and pseudo-random generators (LCG, PCG), and low-discrepancy sequences (Halton). Also, we introduce the KD-tree and coordinate-template acceleration to scale to large models as well as outlier identification to control the error. Experiments on ResNet-18 and Qwen2.5-1.5B/Qwen1.5-7B validate that DS-based compression achieves competitive compression ratios without post-hoc retraining, with controllable decompression error and flexible state-space design, establishing it as a principled and practical compression approach. 
\end{abstract}


\section{Introduction}

In the past decade, the emergence of large-scale pretrained models creates substantial deployment challenges, particularly in resource-constrained settings such as workstations, smartphones, and wearables \cite{zhou2024survey}. Loading LLaMA-2 70B weights in FP16 precision requires approximately 140~GB, nearly six times the memory capacity of an NVIDIA GeForce RTX 4090~\cite{nvidia2022rtx4090}. 
The development of multimodal world models \cite{yin2024survey,zhang2026worldinworld} is likely to further worsen this issue due to the introduction of vision and speech encoders alongside extended context lengths. Consequently, model compression has emerged as a critical strategy for serving large models, aiming to minimize storage requirements without compromising task performance~\cite{zhu2024survey}.
This is quite different from the pre-large-model era, where model sizes remained modest, and the incentives for compression were comparatively weak.

Classical model compression methods fall into four primary categories: pruning, quantization, knowledge distillation, and low-rank decomposition~\cite{wan2024efficient}: i) \textit{Pruning} removes unimportant parameters or structural components ~\cite{frantar2023sparsegpt}, which pursues the parameter sparsity within a model; ii) \textit{Quantization} reduces the numerical precision of weights and activations (e.g., FP32 → INT8/INT4) to decrease storage and bandwidth requirements~\cite{frantar2023gptq}; iii) \textit{Knowledge distillation} transfers knowledge from a large teacher model to a compact student model~\cite{hinton2015distilling, gu2024minillm}, which typically requires additional training, task data, and careful teacher--student design; iv) \textit{Low-rank factorization} approximates weight matrices or tensors using the low-rank decomposition~\cite{sainath2013low}. The achievable compression depends on the intrinsic rank of weights, and can vary substantially across layers, architectures, and tasks. In essence, classical methods exploit different forms of redundancy in the original parameter space such as sparsity, reduced numerical precision, or low-rank structure.

Recently, inspired by the biological genomic bottleneck \cite{shuvaev2024encoding}, a novel basic mechanism for model compression was discovered and engineered in \cite{fan2026hypercompression}, which uses the ergodicity principle that a low-dimensional dynamic system can fill the high-dimensional space. Mathematically, given target weights \(\mathbf{w} = [w_1,\ldots, w_N]^\top\), \cite{fan2026hypercompression} \textbf{employs} irrational winding to construct a function \(h(\theta;n)\) to approximate \textbf{each weight such that} \(w_n \approx h(\theta;n)\), where the number of trajectory parameters \(\theta\) is significantly smaller than the total number of weights (\(\dim(\theta) \ll N\)). This turns model compression into a problem of parsimonious parameter representation. Model compression via dynamic systems is not merely incremental. Instead, it shifts away from decades-old techniques by opening new research directions at the intersection of dynamic systems and efficient deep learning, while also being practically viable for billion-parameter models in resource-constrained settings. Introducing dynamic systems carves out a new research direction in model compression. Nevertheless, several critical theoretical and methodological challenges must be addressed to fully realize the potential of this approach:

Theoretically, \textbf{the relationship between the compression ratio and the error remains unexplored}. Due to the non-locality of the irrational winding, to achieve the same level of error, how large $\theta$ needs to be at least is intractable. Accordingly, the compression ratio is unpredictable. This is a fundamental question in this domain, and solving it will transform the dynamic system from a heuristic into a principled, predictable tool. For example, one can predict to what extent increasing $\theta$ will surely reduce the approximation error, thereby providing the lower bound for the compression ratio and facilitating allocating bits across layers fairly. Methodologically, \textbf{since model compression via dynamic systems is general, can we enrich the armory of this kind by effectively engineering other dynamic systems?} The dynamic system (DS) theory provides a broader design space beyond the irrational winding. Instead of treating irrational winding as the default choice, a broader design space of transformation-induced compression mappings characterized by different coverage, locality, and indexability properties should be examined. 

In this manuscript, we theoretically establish both the asymptotic feasibility and the practical optimizability of DS-based compression. First, the covering guarantee (Theorem \ref{relation_error}) establishes that, under Diophantine conditions, \(M = O(\epsilon^{-(d+\nu)})\) trajectory states form an \(\epsilon\)-net over the weight space, directly linking the compression ratio to the reconstruction error through the covering radius. This justifies why a low-dimensional scalar can approximate high-dimensional weights with bounded distortion. Second, the optimization procedure addresses the discontinuity of the irrational winding by partitioning the domain at all breakpoints; within each continuous segment, the loss reduces to a convex quadratic with a closed-form minimizer, enabling exact global minimization. Together, these results transform the theoretical density guarantee into a practical, deterministic algorithm with a predictable error profile, ensuring that the mechanism is not merely an existence statement. Methodologically, we expand this novel framework from a single ergodic construction to a broader toolbox, \textit{i.e.}, space-filling curves, chaotic dynamics, pseudo-random sequences, and low-discrepancy sequences, which cover representative DS families:

\begin{itemize}
    \item \textit{Space-filling and grid traversal mappings} (\texttt{Hilbert}~\cite{hilbert1891ueber}, \texttt{Peano}~\cite{peano1890courbe}, \texttt{Morton/Z-order}~\cite{morton1966computer}, and \texttt{Snake} traversal): 
    deterministic mappings from compact indices to high-dimensional grid points, offering different locality-preserving structures for grouped weight representation.

    \item \textit{Chaotic-system-based mappings} 
    (\texttt{Lorenz} system~\cite{lorenz1963deterministic}): 
    nonlinear trajectory-based mappings that provide aperiodic coverage of the weight representation space beyond regular grid traversal.

    \item \textit{Congruential and pseudo-random generator mappings} 
    (Linear Congruential Generator (\texttt{LCG})~\cite{lehmer1951mathematical} and Permuted Congruential Generator (\texttt{PCG})): 
    efficient deterministic point generators that produce reproducible pseudo-random sequences with simple indexable structures for scalable compression.

    \item \textit{Quasi-Monte Carlo sequences} 
    (\texttt{Halton}~\cite{halton1960efficiency}): 
    deterministic low-discrepancy sequences that provide uniform coverage for compact weight representation.
\end{itemize}
These families differ in coverage guarantees, computational cost, and approximation accuracy, thereby leading to different compression ratios, speed, and decompression fidelity. The central problem in developing DS-based compression is how to enable fast, accurate, and scalable compression. We introduce effective engineering techniques such as KD-tree to speed up the compression speed and outlier identification to control the error. Particularly, for the space-filling curves, we design the coordinate template method to further accelerate the compression relative to the KD-tree search. 

In brief, our contributions are threefold:

\begin{itemize}
    \item We lay a solid theoretical foundation for model compression based on dynamic systems by illustrating the asymptotic feasibility. To the best of our knowledge, it is the first theoretical result in the direction of DS-based model compression. 
    
    \item We build a general framework for model compression via dynamic systems, engage mainstream dynamic systems as the components, and design engineering techniques to make this framework scalable to large models. 
    
    \item Systematic and comprehensive experiments confirm that DS-based model compression can deliver competitive performance on ResNet and large models such as Qwen2.5-1.5B and Qwen1.5-7B.
\end{itemize}

\section{Related Work}
\label{sec:related}

\subsection{Classical Model Compression Methods}

Model compression aims to minimize storage requirements without incurring substantial performance degradation. Established techniques include pruning, quantization, knowledge distillation, and low-rank decomposition.

\textbf{Pruning} compresses neural networks by removing redundant weights, channels, or structural components. 
Early methods such as Optimal Brain Damage~\cite{lecun1990optimal} and Optimal Brain Surgeon~\cite{hassibi1993optimal} estimate parameter saliency using second-order information, while structured pruning improves hardware compatibility by uniformly removing filters or channels rather than isolated weights~\cite{li2017pruning}. 
For large language models, SparseGPT~\cite{frantar2023sparsegpt} formulates pruning as a sparse regression problem for one-shot compression. Wanda~\cite{sun2024wanda} combines weight magnitude with activation statistics to guide pruning; it requires a small calibration set to estimate activation statistics. 
Recent work further studies structured and task-sensitive LLM pruning, including layer-level pruning~\cite{lu2026reassessing} and reasoning-aware pruning based on chain-of-thought reconstruction~\cite{lucas2025reasoning}. 
Despite these advances, practical speedups from sparsity often depend on specialized sparse kernels, storage formats, and hardware support~\cite{hoefler2021sparsity}.

\textbf{Quantization} reduces storage and memory bandwidth requirements by adopting lower numerical precision for weights and activations. 
For large models, post-training quantization (PTQ) has gained prominence due to its convenience. 
ZeroQuant~\cite{yao2022zeroquant} provides fine-grained weight and activation quantization for large Transformers, while SQuant~\cite{guo2022squant} introduces a data-free PTQ through a diagonal-Hessian approximation. SmoothQuant~\cite{xiao2023smoothquant} applies activation-to-weight smoothing transformations to mitigate activation outliers.

For weight-only compression, GPTQ~\cite{frantar2023gptq} employs approximate Hessian-based layer-wise reconstruction to minimize the error; AWQ~\cite{lin2024awq} protects salient weight channels through activation awareness.
SpQR~\cite{dettmers2024spqr} retains quantization-sensitive outliers in higher precision, while QuIP~\cite{chee2023quip} improves low-bit quantization through incoherence processing.

\textbf{Knowledge Distillation}
transfers knowledge from a large teacher model to a compact student model~\cite{hinton2015distilling}, with extensions to intermediate feature matching~\cite{romero2015fitnets} and relation transfer~\cite{park2019relational}. For large language models, MiniLLM~\cite{gu2024minillm} adopts reverse KL divergence to improve distribution matching, while recent work extends knowledge distillation to pre-training~\cite{gu2025kdpretraining} and improves teacher--student alignment through speculative training or score-matching objectives~\cite{xu2025speculativekd,kim2026concretescore}.
However, distillation typically requires additional training and task data, and its effectiveness depends on student capacity, teacher quality, and teacher--student distribution alignment~\cite{stanton2021does}.

\textbf{Low-Rank Decomposition} approximates weight matrices using products of lower-dimensional factors~\cite{sainath2013low}, while tensor decomposition extends this idea to higher-order parameter structures~\cite{novikov2015tensorizing}. 
Parameter-efficient fine-tuning methods such as LoRA~\cite{hu2022lora} further suggest that large-model adaptations often lie in low-dimensional subspaces. 
Recent LLM compression methods increasingly use low-rank components as residual or subspace corrections, including low-precision and low-rank decomposition~\cite{saha2024caldera}, low-rank approximation combined with quantization and sparsity~\cite{mozaffari2025slim}, and low-rank residuals for mixed-precision quantization~\cite{saxena2025resq}.
However, the achievable compression depends on the intrinsic rank structure of the weights, and can vary substantially across layers, architectures, and tasks.

Here, Our work falls under the umbrella of model compression based on dynamic systems, which is orthogonal to these classical methods by representing model weights through a compact representation.

\subsection{Model Compression via dynamic systems}

Motivated by ergodic theory~\cite{walters1982introduction}, Fan \textit{et al.} \cite{fan2026hypercompression} introduce a distinct model compression paradigm by employing irrational winding to generate dense trajectories over the parameter domain and report approximately $2.6\times$ compression on LLaMA2-7B. \cite{fan2025compresssam} applies the irrational winding-based model compression to SAM and its variants, demonstrating its applicability to large segmentation foundation models. 
The method introduces HyperLinear, a dedicated operator that fuses decompression with matrix multiplication to improve the inference efficiency of compressed SAMs. COLI~\cite{wang2025coli} incorporates the irrational winding-based model compression into an INR-based large-image compression framework, where NeRV represents images as neural network parameters, and DS is used to further compress the resulting model weights. 
This shows how DS can serve as a plug-and-play post-training parameter-compression module beyond standard classification and language-model settings.

However, many dynamic systems conform to ergodicity. Irrational winding provides a mathematically tractable instantiation, and the existence of better dynamic systems is possible. This leaves a broad design space of DS-induced compression mappings largely unexplored.
Here, we analyze and compare a broad family of dynamic systems under a unified theoretical and empirical framework as illustrated in Figure \ref{fig:distributions}. 

\subsection{Dynamic Systems and Deterministic Mappings}
\begin{figure}[t]
    \centering
    \includegraphics[width=\linewidth]{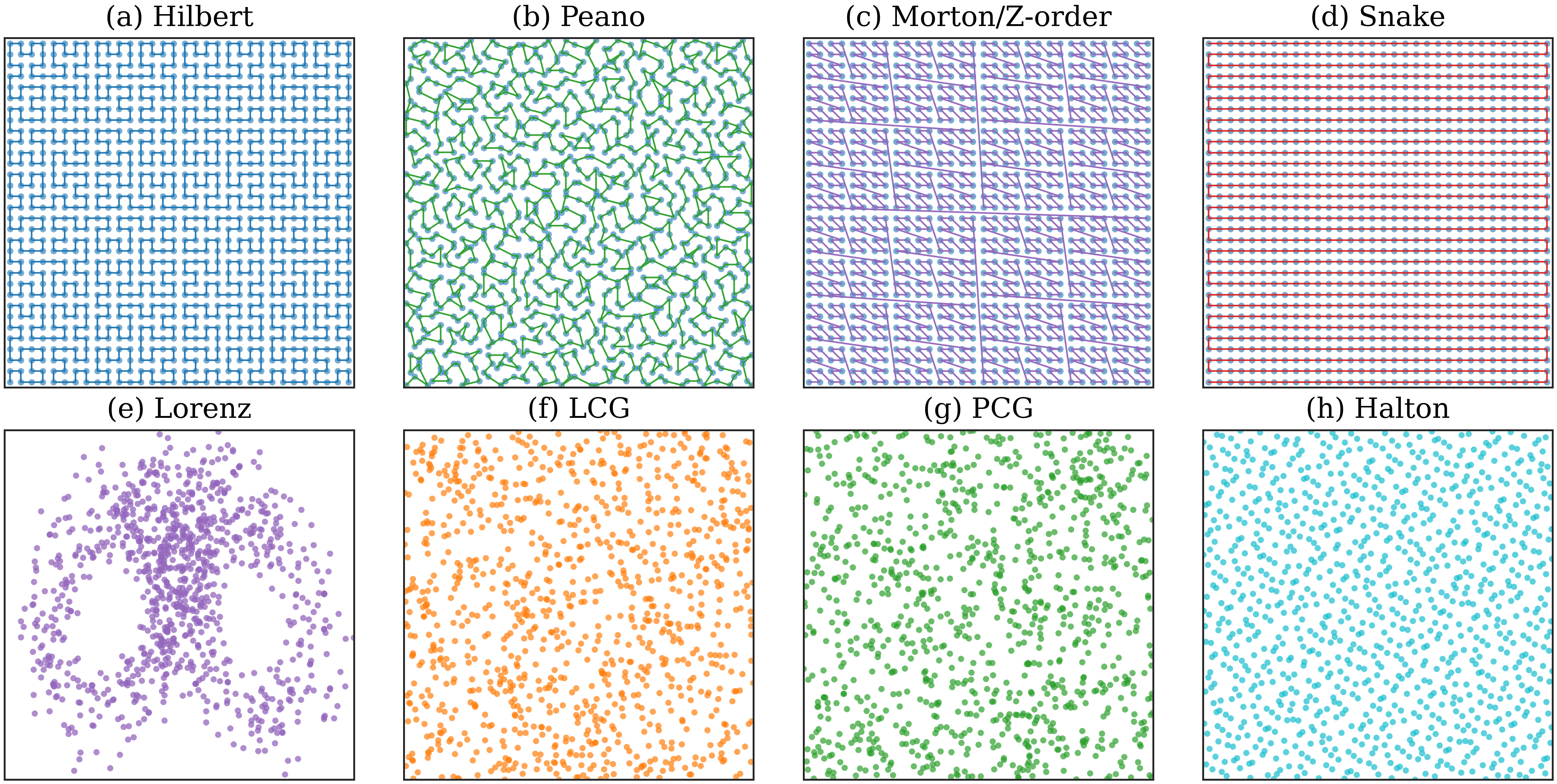}
    \caption{Representative geometries generated by eight evaluated mapping families. 
    \texttt{Hilbert}, \texttt{Peano}, \texttt{Morton/Z-order}, and \texttt{Snake} produce structured lattice traversals with different locality characteristics and traversal continuity. \texttt{Lorenz} generates a non-uniform attractor-induced distribution concentrated on a low-dimensional manifold. 
    \texttt{LCG} and \texttt{PCG} produce deterministic pseudo-random point sets with different correlation structures, while \texttt{Halton} generates low-discrepancy coverage with reduced large-scale gaps. 
    All methods are visualized after normalization to the same local domain.
    }
\label{fig:distributions}
\vspace{-0.6cm}
\end{figure}

Dynamic systems describe how the state of a system evolves according to a fixed rule. 
In mathematical terms, an evolution map \(\Phi(t,\x):T\times \mathcal{X}\rightarrow \mathcal{X}\) maps an initial state \(\x\in \mathcal{X}\) in a phase space \(\mathcal{X}\) to its future state under the evolution parameter \(t\). 
This formulation provides a compact way to generate structured trajectories or point sets. Common dynamic systems include space-filling curves, chaotic systems, pseudo-random generators, and low-discrepancy sequences. As illustrated in Figure~\ref{fig:distributions}, 
i) space-filling curves, such as \texttt{Hilbert}, \texttt{Peano}, and \texttt{Morton/Z-order} mappings, define structured traversals between the one-dimensional index and multi-dimensional coordinates, often with locality-preserving behaviors. 
\texttt{Snake} ordering is a simpler grid traversal based on deterministic scanning. ii) Pseudo-random generators, including linear congruential generators (\texttt{LCG}) and permuted congruential generators (\texttt{PCG}), produce deterministic sequences with compact mathematical descriptions. iii) Chaotic systems such as the \texttt{Lorenz} system generate trajectory-based point sets with nontrivial phase-space structure. iv) Low-discrepancy sequences, such as \texttt{Halton}, provide more uniform domain coverage than random samples, and are widely used in quasi-Monte Carlo methods (\texttt{QMC}).

These dynamic systems induce a state set with different locality, coverage, regularity, and indexing properties, making them natural candidates for studying DS-based model compression. Our work marries the dynamic system with model compression by exploring the possibility of using these dynamic systems as a model compression tool, which extends the application boundary of dynamic systems into AI, and shall bring many new research opportunities. 

\section{Theory}
\label{sec:theory}
  
\begin{proposition}[\cite{katok1995introduction}]
\label{prop1}
Suppose $a_1,\cdots,a_N$ are rationally independent, 
for any given set $\{w_n\}_{n=1}^N \subseteq [0,1]$ and $\epsilon > 0$, there exists $\theta^* \in \mathbb{R}_+$ such that
\begin{equation}
    |w_n-\tau(\theta^*a_n)|<\epsilon,\; n=1,\cdots,N,
\label{eqn:core} 
\end{equation} where $\tau(z)=z-\lfloor z \rfloor$.
\label{ir}
\end{proposition}

Intrinsically, earlier work \cite{fan2026hypercompression} just engineers Proposition \ref{ir} and shows the feasibility of model compression based on dynamic systems. Here, we elaborate on the relationship between $\theta$ and $\epsilon$ under mild conditions, which is a fundamental theoretical issue regarding using the irrational winding to do model compression. 

We use the following standard definitions:

\begin{itemize}
    \item For \(\x = (x_1, \dots, x_d) \in \R^d\), the sup-norm is
    \[
    \norm{\x}_{\infty} = \max_{1 \le j \le d} |x_j|.
    \]
    \item For \(z \in \R\), the torus distance to the nearest integer is
    \[
    \norm{z}_{\T} = \min_{m \in \Z} |z - m|.
    \]
    For vectors \(\z \in \R^d\), we take the coordinate-wise torus distance:
    \[
    \norm{\z}_{\T^d} = \max_{1 \le j \le d} \norm{z_j}_{\T},
    \] where $\T^d = \mathbb{R}^d / \mathbb{Z}^d$, which is a torus meaning that two vectors \(\x, \y \in \mathbb{R}^d\) are identified if and only if their difference is an integer vector.
\end{itemize}

\begin{lemma}[Density of the $\epsilon$-net]
Let \(K > 0\), and define
\begin{equation}
\Delta(K) := \min_{0<|\k|\le K} \norm{\k\cdot \a}_{\T},
\end{equation}
where $\a=[a_1,a_2,\cdots,a_d]$ and \(|\k| = \sum_{j=1}^d |k_j|\). There exists a constant \(C_d\) depending only on the dimension \(d\) such that if
\begin{equation}
M \ge C_d \, \epsilon^{-d} \, \Delta(K)^{-1}, \qquad K \asymp \epsilon^{-1},
\end{equation}
then the finite set \(\{m\a : 0 \le m \le M\}\) is an \(\epsilon\)-net in \(\T^d\), i.e., every \(\bf{\beta} \in \T^d\) is within distance \(< \epsilon\) of some point in the set.
\label{density}
\end{lemma}

\begin{proof}

We argue by contradiction. Suppose there exists a cube \(Q\) of radius \(\epsilon\) centered at some \(\beta \in \T^d\) that contains no point \(m\a\) for \(0 \le m \le M\). Choose a non-negative smooth function \(\phi: \T^d \to \R_{\ge 0}\) such that
\begin{enumerate}
    \item \(\phi(\x) \ge 1\) on the sub-cube \(Q_{\epsilon/2}\) (radius \(\epsilon/2\));
    \item \(\operatorname{supp}(\phi) \subset Q\) (so \(\phi\) vanishes outside \(Q\));
    \item The integral satisfies \(\widehat{\phi}(0) := \int_{\T^d} \phi(\x) \, d\x \asymp \epsilon^d\).
\end{enumerate}
Such a function exists (e.g., a standard smooth bump function supported in \(Q\)).

Since the orbit points \(\{m\a\}_{m=0}^M\) avoid \(Q\), and \(\phi\) vanishes outside \(Q\), we have the exact identity:
\begin{equation}
\sum_{m=0}^{M} \phi(m\a) = 0.
\label{phi_support}
\end{equation}
Using a Fejér kernel truncation \cite{coroianu2017lp}, for any arbitrarily small \(\delta > 0\), there exists a trigonometric polynomial
\begin{equation}
P(\x) = \sum_{|\k| \le K} \widehat{P}(\k) e^{2\pi i \k \cdot \x}
\end{equation}
with frequencies bounded by \(K\) (where \(K \asymp \epsilon^{-1}\)) such that
\begin{equation}
\sup_{\x \in \T^d} |P(\x) - \phi(\x)| \le \delta. 
\label{approximation}
\end{equation}
Moreover, we can ensure the following standard properties:
\begin{equation}
\widehat{P}(0) \asymp \epsilon^d, \qquad |\widehat{P}(k)| \le \widehat{P}(0) \quad \text{for all } k. 
\end{equation}

Combining Eqs. \eqref{phi_support} and \eqref{approximation}, we slightly abuse the notation by writing
\begin{equation}
\sum_{m=0}^{M} P(m\a) = 0.
\end{equation}
On the other hand, expand the Fourier series:
\begin{equation}
\sum_{m=0}^{M} P(m\a) = \sum_{\k \in \Z^d} \widehat{P}(\k) \sum_{m=0}^{M} e^{2\pi i m (\k\cdot \a)}.
\end{equation}

Separating this term from the rest gives
\[
\sum_{m=0}^{M} P(m\a)
= \widehat{P}(0) (M+1) + \sum_{0 < |\k| \le K} \widehat{P}(\k) \sum_{m=0}^{M} e^{2\pi i m (\k\cdot \a)}.
\]
Since \(\widehat{P}(0)M \le \widehat{P}(0)(M+1)\), we have
\[
M \widehat{P}(0) \le (M+1) \widehat{P}(0) = \left| \sum_{0 < |\k| \le K} \widehat{P}(\k) \sum_{m=0}^{M} e^{2\pi i m (\k\cdot \a)} \right|.
\]

Finally, applying the triangle inequality to the finite sum on the right, we arrive at the key inequality:
\begin{equation}
M \widehat{P}(0) \le \sum_{0 < |\k| \le K} |\widehat{P}(\k)| \left| \sum_{m=0}^{M} e^{2\pi i m (\k\cdot \a)} \right|.
\label{key_eq}
\end{equation}

Eq. \eqref{key_eq} is the rigorous starting point for the rest of the proof. The crucial step here is the use of the \emph{reverse triangle inequality}, which correctly transfers the negligible approximation error \(\delta\) from the left side to the right side as an additive error term, rather than incorrectly assuming \(\sum P=0\).

Using the geometric series estimate:
\[
\left| \sum_{m=0}^{M} e^{2\pi i m \theta} \right| \le \min\left(M+1,\ \frac{1}{2\norm{\theta}_{\T}}\right) \le \frac{1}{\norm{\theta}_{\T}},
\]
and \(|\widehat{P}(k)| \le \widehat{P}(0)\), we obtain
\[
M \widehat{P}(0) \le \widehat{P}(0) \sum_{0<|\k|\le K} \frac{1}{\norm{\k\cdot \a}_{\T}}.
\]
Cancel \(\widehat{P}(0)\). By the definition of \(\Delta(K)\), for all \(0<|\k|\le K\), we have \(\norm{\k\cdot \a}_{\T} \ge \Delta(K)\). Therefore
\[
M \le \sum_{0<|\k|\le K} \frac{1}{\norm{\k\cdot \a}_{\T}} \le \frac{\#\{k:0<|\k|\le K\}}{\Delta(K)} \le \frac{C_d' K^d}{\Delta(K)}.
\]
Since \(K \asymp \epsilon^{-1}\), we have \(K^d \le C_d'' \epsilon^{-d}\). Hence,
\[
M \le C_d''' \, \epsilon^{-d} \, \Delta(K)^{-1}.
\]
This contradicts the lemma's assumption \(M \ge C_d \, \epsilon^{-d} \, \Delta(K)^{-1}\) (choosing \(C_d > C_d'''\)). Thus the contradiction is established, and the lemma follows. 
\end{proof}

\begin{theorem}
Let \(\a \in \R^d\) satisfy the linear form Diophantine condition: there exist constants \(c>0, \nu>0\) such that for all nonzero integer vectors \(\k \in \Z^d \setminus \{0\}\),
\[
\norm{\k \cdot \a}_{\T} \ge c |\k|^{-\nu}.
\]
Then there exist constants \(C>0\) and \(\epsilon_0>0\) such that for every \(0<\epsilon<\epsilon_0\) and every \(\beta \in \T^d\), there exists an integer \(m\) satisfying
\[
0 \le m \le C \epsilon^{-(d+\nu)},\qquad \norm{m\a - \beta}_{\T^d} < \epsilon.
\]
\label{relation_error}
\end{theorem}

\begin{proof}
By the Diophantine condition, taking \(K \asymp \epsilon^{-1}\), we have
\[
\Delta(K) = \min_{0<|\k|\le K} \norm{\k\cdot \a}_{\T} \ge \min_{0<|\k|\le K} c |\k|^{-\nu} \ge c K^{-\nu}.
\]
Since \(K \asymp \epsilon^{-1}\), there exists a constant \(c'>0\) such that
\[
\Delta(K) \ge c' \epsilon^{\nu}.
\]

Substituting this lower bound into Lemma \ref{density}, there exists a constant \(C>0\) such that whenever
\[
M \ge C \epsilon^{-d} \cdot (\epsilon^{\nu})^{-1} = C \epsilon^{-(d+\nu)},
\]
the point set \(\{m\a:0\le m\le M\}\) forms an \(\epsilon\)-net.

Hence, for any \(\beta \in \T^d\), there must exist \(0 \le m \le M\) satisfying:
\[
\norm{m\a - \beta}_{\T^d} < \epsilon.
\]
Choosing the constant \(C\) according to Lemma \ref{density} concludes the proof. 

\end{proof}

\textbf{Remark 1}. We argue that Theorem \ref{relation_error} holds in a general sense. The linear form Diophantine condition is by no means an exotic or artificially restrictive assumption. Common irrational numbers can satisfy this condition. For instance, the classical two-dimensional vector \(\a=(\sqrt{2}, \sqrt{3})\) satisfies the condition with the exponent \(\nu = 2\), as a direct consequence of the Schmidt Theorem \cite{schmidt1980diophantine}. Thus, rather than being a pathological rarity, the linear form Diophantine condition serves as a sharp but broadly applicable threshold.

\section{A Unified Framework}
\label{sec:methodology}

\subsection{Dynamic System Families}
\label{subsec:mapping_families}

Using a dynamic system converts a high-dimensional vector into low-dimensional trajectory parameters. Early work \cite{fan2026hypercompression} has shown that the irrational winding can indeed compress large models well.
In this work, we examine four common families of dynamic systems: space-filling/grid traversal mappings, chaotic trajectory-based mappings, congruential and pseudo-random generator mappings, and quasi-Monte Carlo sequence mappings. In doing so, we can enrich the armory of model compression via dynamic systems. Because there is no definitive theory on the superiority of dynamic systems, involving more tools can increase the likelihood of achieving better compression performance.   

\textbf{Space-filling curves}
 $(\mathtt{SFC})$ define deterministic mappings from a one-dimensional index to a \(D\)-dimensional lattice point:
\begin{equation}
    \Phi_{\mathtt{SFC}}:
    \{0,\ldots,S-1\}
    \rightarrow
    \{0,\ldots,m-1\}^{D},
    \qquad
    S=m^D,
    \label{eq:grid_mapping}
\end{equation}
where \(m\) denotes the per-axis resolution, and \(\mathtt{SFC}\) is the instantiated specific space-filling method. 

Representative space-filling curves are \texttt{Hilbert}, \texttt{Peano}, \texttt{Morton/Z-order}, and \texttt{Snake} traversals, as shown in Figure \ref{fig:distributions}(a)-(d).
Each dynamic system $\Phi$ can generate a trajectory 
\begin{equation}
    \Phi_{\mathtt{SFC}}(\xi),
    \quad\xi \in [0,S-1].
\end{equation}
For \texttt{Hilbert}, \texttt{Morton/Z-order}, and \texttt{Snake},
we use \(m=2^k\), whereas \texttt{Peano} uses \(m=3^k\).
Each mapping provides a deterministic bijection between the scalar
indices and the corresponding lattice points.
Detailed constructions of the four mappings are provided in the
Supplementary Material.

\textbf{Chaotic systems} can generate deterministic trajectories with complex aperiodic behavior under chaotic parameter settings. As a representative case, 
we instantiate this family with the \texttt{Lorenz} system:
\begin{equation}
    \dot{x}=\sigma(y-x),\quad
    \dot{y}=x(\rho-z)-y,\quad
    \dot{z}=xy-\gamma z .
    \label{eq:lorenz}
\end{equation}
where $\sigma$, $\rho$, and $\gamma$ are system parameters. With system parameters, an initial condition $(x_0,y_0,z_0)$, and an integration step size $\Delta t$, we numerically integrate Eq.~\eqref{eq:lorenz}.
After discarding the first $T_{\mathrm{trans}}$ transient states, we retain $S$ consecutive post-transient states:
\begin{equation}
\mathbf{h}_{\xi}
=
\bigl(x_{\xi},y_{\xi},z_{\xi}\bigr)
\in \mathbb{R}^{3},
\qquad
\xi=0,\ldots,S-1 .
\label{eq:lorenz_state}
\end{equation}
Here, $\xi$ denotes the retained trajectory index.

For the two-dimensional compression setting, the post-transient \texttt{Lorenz} trajectory is projected to a two-dimensional space by
\begin{equation}
\Phi_{\mathrm{Lorenz}}(\xi) 
=
P_{\mathrm{2D}} \left( \mathbf{h}_{\xi} \right) \in \mathbb{R}^{2}, 
\qquad
\xi=0,\ldots,S-1 , 
\label{eq:lorenz_projection} 
\end{equation}
where $P_{\mathrm{2D}} $ denotes the projection operator from the three-dimensional \texttt{Lorenz} states to the two-dimensional template space.
In our implementation, $P_{\mathrm{2D}}$ is obtained by PCA on the post-transient \texttt{Lorenz} states.

As shown in Figure \ref{fig:distributions}(e), unlike grid-based traversals, \texttt{Lorenz}-based systems do not lie on a regular lattice. 
Unlike \texttt{QMC} and congruential generators, it is not designed to uniformly cover the canonical domain.
Although the states are uniform in integration time, their spatial distribution follows the geometry and visiting density of the \texttt{Lorenz} attractor.
Under chaotic parameter settings, the \texttt{Lorenz} trajectory is confined to its attractor after the transient phase and repeatedly visits different regions of the attractor.
Thus, the retained post-transient samples can provide a dense empirical coverage of the attractor support as the trajectory length increases.

\textbf{Congruential and pseudo-random generators}
can be viewed as deterministic recursive systems that map a scalar index to a multi-dimensional system state. As shown in Figure \ref{fig:distributions}(f)-(g), 
unlike grid-based traversals, systems do not explicitly enumerate a dense grid. 
Instead, a \(D\)-dimensional system point is obtained by unfolding a one-dimensional recurrence for $D$ consecutive steps.

For \texttt{LCG}, the internal state follows the recurrence
\begin{equation}
    z_{t+1} = (a z_t + b) \bmod M ,
    \label{eq:lcg}
\end{equation}
where \(z_t\in\{0,\ldots,M-1\}\) denotes the integer state at recursion step \(t\), and \(a\), \(b\), and \(M\) denote the multiplier, increment, and modulus. 
Given the size \(S\) for the candidate set, we generate \(S+D-1\) scalar states and form the \(\xi\)-th \(D\)-dimensional template by a sliding window:
\begin{equation}
\Phi_{\mathrm{LCG}}(\xi)
=\bigl(
z_{\xi},
\ldots,
z_{\xi+D-1}
\bigr)
\in \{0,\ldots,M-1\}^{D},
    \label{eq:lcg_recursive}
\end{equation}
where \(\xi\in\{0,1,\ldots,S-1\}\) denotes the index.

For \texttt{LCG}, we set \(M=2^k\) and use full-period parameters satisfying the Hull--Dobell conditions~\cite{hull1962random}.
This guarantees a full-period state transition for each modulus \(M\).

\texttt{PCG} uses the same congruential state transition as Eq.~\eqref{eq:lcg}, but does not directly induce the raw state \(z_t\). Instead, it applies a deterministic output permutation:
\begin{equation} 
    y_t 
    =
    \mathcal{P}(z_t), 
    \label{eq:pcg_output_permutation}
\end{equation}
where $ y_t\in\{0,\ldots,M-1\}$, and \(\mathcal{P}(\cdot)\) denotes a deterministic output transformation composed of bit mixing and state-dependent rotation operations. 
In our implementation, \(\mathcal{P}(\cdot)\) is a PCG-style output permutation as
\begin{equation}
\mathcal{P}(z)
=
\operatorname{xorshift}_{3}^{R}
\left(
\operatorname{rotr}_{k}
\left(
\mu \cdot
\operatorname{xorshift}_{5}^{R}(z),
\rho
\right)
\right),
\label{eq:pcg_permutation_compact}
\end{equation}
where
\begin{equation}
\operatorname{xorshift}_{q}^{R}(x)
=
x \oplus (x \gg q).
\label{eq:xorshift_right}
\end{equation}
Here, \(\oplus\) denotes bitwise XOR, \(\gg\) denotes right shift, \(\operatorname{rotr}_{k}(\cdot,\rho)\) denotes circular right rotation over a \(k\)-bit word, \(\mu\) is an odd multiplication constant, and \(\rho\) is a fixed rotation offset.
All operations are performed within the \(k\)-bit integer space.
Since xorshift, circular rotation, and multiplication by an odd integer are invertible over \(k\)-bit words, \(\mathcal{P}\) defines a deterministic bijection on \(\{0,\ldots,M-1\}\), where \(M=2^k\).

Using the emitted sequence \(\{y_t\}_{t\geq 0}\), PCG adopts the same sliding-window template construction as LCG.
The coordinate functions of the PCG system are then defined as 
\begin{equation}
\Phi_{\mathrm{PCG}}(\xi)
=
\bigl(
y_{\xi},
\ldots,
y_{\xi+D-1}
\bigr)
\in \{0,\ldots,M-1\}^{D},
\label{eq:pcg_recursive}
\end{equation}
where \(\xi\in\{0,1,\ldots,S-1\}\) denotes the template index, and \(S\) is the number of states.

Therefore, \texttt{LCG} and \texttt{PCG} share the same congruential state transition and sliding-window construction.
The difference lies in the sequence: \texttt{LCG} directly uses the raw state \(z_t\), whereas \texttt{PCG} uses the permuted output \(y_t=\mathcal{P}(z_t)\).
This output permutation changes the structure of the state set while preserving deterministic generation and efficient reconstruction.

\textbf{Quasi-Monte Carlo (QMC) sequences}
are deterministic low-discrepancy point sets that provide structured coverage of the local representation domain, as shown in Figure \ref{fig:distributions}(h). 
Compared with pseudo-random generators, \texttt{QMC} sequences are designed to reduce discrepancy, meaning that their empirical distribution better matches the target domain volume over axis-aligned regions. 
This makes them suitable for constructing relatively uniform states under a finite sampling budget.

We instantiate \texttt{QMC} mappings with the \(D\)-dimensional \texttt{Halton} sequence.
Let \(j=0,\ldots,S-1\) denote the candidate index and set \(\xi=j+1\) as the \texttt{Halton} sequence index. 
Let \((b_1,\ldots,b_D)\) be the coordinate bases, which are chosen as the first \(D\) prime numbers in our implementation.
The \texttt{Halton} state is defined as
\begin{equation}
\Phi_{\mathrm{Halton}}(\xi)
=
\bigl(
\varphi_{b_1}(\xi),
\varphi_{b_2}(\xi),
\ldots,
\varphi_{b_D}(\xi)
\bigr)
\in [0,1)^D ,
\label{eq:halton}
\end{equation}
where \(\varphi_b:\mathbb{N}\to[0,1)\) is the radical-inverse function in base \(b\). 
Specifically, if the base-$b$ expansion of $\xi$ is
\begin{equation}
\xi
=
\sum_{r=0}^{m}
a_r^{(b)}(\xi)b^r,
\qquad
a_r^{(b)}(\xi)\in\{0,\ldots,b-1\},
\end{equation}
then
\begin{equation} 
    \varphi_b(\xi) 
    =
    \sum_{r=0}^{\infty} \frac{a_r^{(b)}(\xi)}{b^{r+1}}.
    \label{eq:radical_inverse} 
\end{equation}
Here, \(a_r^{(b)}(\xi)\) denotes the \(r\)-th digit of \(\xi\) in its base-\(b\) expansion.
For example, when $D=2$ used in our implementation, the bases are $(b_1,b_2)=(2,3)$, giving
\begin{equation}
    \Phi_{\mathrm{Halton}}(\xi)=(\varphi_2(\xi),\varphi_3(\xi)).
\end{equation}

\subsection{The Framework of Model Compression via Dynamic Systems}
\label{subsec:implementation}

As Figure \ref{fig:overview} shows, this section describes the unified model compression framework based on dynamic systems. In this compression framework, each dynamic system serves as a plug-and-play component. 

\begin{figure}[t]
    \centering
    \includegraphics[width=\linewidth]{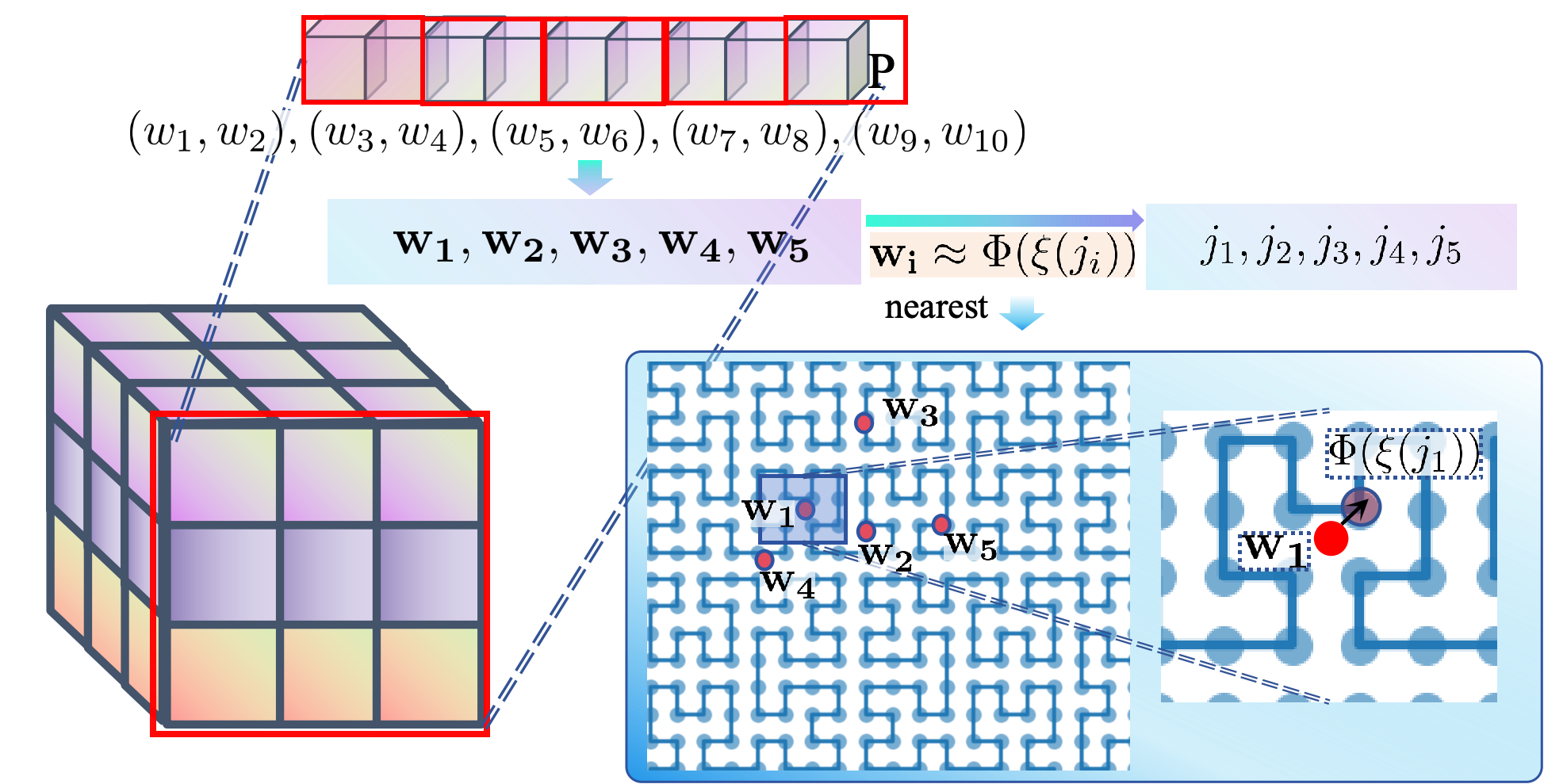}
\caption{
An overview of the proposed DS-based compression framework. The parameter tensor is flattened and partitioned into grouped weight vectors, which are mapped into a tensor-wise local domain. 
A dynamic system generates a finite trajectory of states, and each grouped vector is assigned to its nearest state and the associated index. 
The model is compressed by storing the resulting index instead of the original floating-point vectors, while the states are recovered from the index for decompression.
}
\label{fig:overview}
\vspace{-0.5cm}
\end{figure}

\textbf{Flatten}: Let \(\boldsymbol{\Theta}=\{\W^{(\ell)}\}_{\ell=1}^{L}\) denote the set of parameter tensors across $L$ layers in a pretrained neural network. 
For a tensor \(\W^{(\ell)}\), let \(N_\ell=\mathtt{numel}(\W^{(\ell)})\) denote the total number of weights. 
We flatten each tensor into a vector:
\begin{equation}
    \tilde{\w}^{(\ell)}
    =
    \mathtt{Flatten}\left(\W^{(\ell)}\right)
    \in\mathbb{R}^{N_\ell}.
\end{equation}

\textbf{Partition}: Since $N_l$ is huge, it is impossible for the algorithm to deal with $\tilde\w^{(l)}$ directly. Therefore, we use the divide-and-conquer strategy by partitioning $\tilde\w^{(l)}$ into small \(D\)-dimensional vectors. $\tilde\w^{(l)}$ is padded if necessary, so that the number of scalar entries is divisible by \(D\). The padded entries are discarded after decompression. Mathematically, suppose we have $M_\ell=\left\lceil \frac{N_\ell}{D}\right\rceil$ groups, 
\begin{equation}
    \left\{
    \bar{\w}^{(\ell)}_i
    \right\}_{i=1}^{M_\ell} = \mathtt{Partition}(\tilde{\w}^{(l)}),
    \qquad
    \bar{\w}^{(\ell)}_i \in \mathbb{R}^{D}
    .
    \label{eq:grouping}
\end{equation}
Each group contains \(D\) consecutive scalar parameters from the flattened tensor. 
Padded entries are recorded in the metadata and discarded after decompression.

\textbf{Shift and Scale}: Since different weights may have different magnitudes, we shift and rescale the vector $\bar{\w}$ into a hypercube, followed by a unified operation. 
For each weight vector \(\bar{\w}^{(\ell)}\), we apply a shift and scale transformer $\mathtt{ShiftScale}: \mathbb{R}^{D}\rightarrow [c^{(\ell)}-\frac{B_\ell}{2},c^{(\ell)}+\frac{B_\ell}{2}]^D$:
\begin{equation}
    \w^{(\ell)}_i
    =
    \mathtt{ShiftScale}
    \left(
    \bar{\w}^{(\ell)}_i
    \right),
    \label{eq:local_domain_transform}
\end{equation}
where $c^{(l)}$ is the geometric center \(\mathbf{c}^{(\ell)}\)  of all vectors \(\{\bar{\w}^{(\ell)}_i\}_{i=1}^{M_\ell}\), and $B_l$ is the size of the hypercube. Both $c^{(l)}$ and $B_l$ can be set to different values with respect to different layers. Hereafter, we stop using the superscript or subscript $l$ for convenience.

The trajectory $\Phi(\xi)$ generated by a dynamic system is in $[c-\frac{B}{2},c+\frac{B}{2}]^D$. Mathematically, we have for any $\w \in [c-\frac{B}{2},c+\frac{B}{2}]^D$, there exists $\xi^*$ such that
\begin{equation}
    \Vert\Phi(\xi^*)-\w\Vert \leq \epsilon.
\end{equation}
According to this property, for each dynamic system, the compression is essentially realized by 
\begin{equation}
    \min_\xi \Vert \Phi(\xi) - \w\Vert,
    \label{key_equation}
\end{equation}
which turns a vector $\w$ into $\xi$. In real world,  $\xi$ itself is discrete like the space-filling curves; otherwise, we sample $\xi$ into a table and pre-store this table to avoid solving Eq. \eqref{key_equation} analytically. Without doing so, the entire compression will be time-consuming and unable to scale to large models. We denote the discrete values of $\xi$ as $\xi_j$, $j=0,\ldots,S-1$, where \(S\) denotes the number of discrete values. For each \(\w\), we search along $\xi$ for the state $\Phi(\xi)$ that can best represent it. 
Then, we have
\begin{equation}
    \Phi_{\mathcal{T}}(\mathbf{\xi }_j)
    =\w'
    \in [c-\frac{B}{2},c+\frac{B}{2}]^D.
    \label{eq:dynamic_operator}
\end{equation}
Now, we analyze how Eqs. \eqref{key_equation} and \eqref{eq:dynamic_operator} are implemented in a specific dynamic system:
\begin{itemize}
\item \textbf{Space-filling curves}: Including \texttt{Hilbert}, \texttt{Peano}, \texttt{Morton/Z-order}, and \texttt{Snake}, $\mathtt{SFC}$ first quantizes \(\w\) onto a regular lattice.
We set the per-axis grid resolution  \(m\) and the grid-cell side length \(s\) as
\begin{equation}
m=\sqrt[D]{S},\qquad  s=\frac{B}{m}.
\end{equation}
For \(\w\), we divide it by \(s\) to locate its grid cell.
Then \(\w\) is assigned to the geometrical nearest lattice point as
\begin{equation}
\mathbf{w}'
=
\operatorname{clip}
\left(
\left\lfloor
\frac{\mathbf{w}}{s}
\right\rfloor
+
\mathbb{I}
\left[
(\mathbf{w}\bmod s)
\ge
\frac{s}{2}
\right],
0,
m-1
\right),
\label{sfc_integer}
\end{equation}
where \(\mathbb{I}[\cdot]\) is applied element-wise and \(\operatorname{clip}(\cdot,0,m-1)\) ensures that each coordinate remains within the valid lattice-index range.
Then, we can use the explicit relationship according to Supplementary Eqs.~S2,~S5,~S9, and~S12 to obtain $\xi$ from $\w'$. The generated mapping
\begin{equation}
\Phi_{\mathtt{SFC}}: \{0,\ldots,m^D-1\} \rightarrow \{0,\ldots,m-1\}^{D} \end{equation}
is a bijection over the finite lattice. Therefore, the corresponding \(\xi\) is computed by the inverse traversal rule:
\begin{equation}
    \xi 
    =
    \Phi_{\mathtt{SFC}}^{-1}(\mathbf{w}'), 
    ~~
    \mathtt{SFC}
    \in \{\mathrm{Hilbert},\mathrm{Peano},\mathrm{Morton},\mathrm{Snake}\}. \label{eq:grid_inverse_state} 
\end{equation}
For \texttt{Morton/Z-order} and \texttt{Snake} traversal, \(\xi\) can be computed explicitly from the coordinate rules, while for \texttt{Hilbert} and \texttt{Peano} traversals, it is obtained by reversing the corresponding recursive construction.

\item \textbf{Pseudo-random (PR)}: The scalar state \(\xi\) is not obtained by an explicit inverse traversal rule. Instead, we apply Eqs.~\eqref{eq:lcg_recursive} and ~\eqref{eq:pcg_recursive} to generate a state set 
\begin{equation}
    \mathcal{C}_{\mathtt{PR}} 
    =
    \left\{ \Phi_{\mathtt{PR}}(\xi_i) \; 
    \middle| \; i\in(0,\dots,S-1)\right\},
\end{equation}
where $ \mathtt{PR}
    \in
    \{\texttt{LCG},\texttt{PCG}\}$, and 
\(\xi\) is selected by the nearest-neighbor assignment: 
\begin{equation}
    \xi^{*} 
    =
    \arg\min_{\mathcal{C}_{\mathtt{PR}} }
    \left\| 
    \Phi_{\mathtt{PR}}(\xi) 
    - \mathbf{w} 
    \right\| .
\end{equation}
Unlike \texttt{SFC}, this assignment is not an inverse of a lattice bijection, but a nearest-state search over a finite deterministic candidate set.

\item \textbf{QMC}: \(\xi\) indexes a deterministic low-discrepancy point.
Since \texttt{Halton} does not provide a closed-form inverse for \(\mathbf{w}\), we apply Eq.~\eqref{eq:halton} to construct the finite state set:
\begin{equation}
    \mathcal{C}_{\mathrm{QMC}} 
    = 
    \left\{ 
    \Phi_{\mathtt{Halton}}(\xi_i) \; 
    \middle| \; i\in(0,\dots,S-1) 
    \right\}, 
\end{equation}
and \(\xi\) is then computed as 
\begin{equation}
    \xi^{*}
    = 
    \arg\min_{\mathcal{C}_{\mathtt{QMC}} }
    \left\| 
    \Phi_{\mathrm{Halton}}(\xi) - \mathbf{w} \right\| . 
    \label{eq:qmc_inverse_state} 
\end{equation}
This assigns \(\mathbf{w}\) to the nearest index \(\xi^*\).

\item \textbf{Chaos}:
    \(\xi\) is also sampled from a deterministic trajectory rather than from a regular lattice.
    For the \texttt{Lorenz} instantiation, since the trajectory winds around the attractor, we align a circular region instead of the default hypercube.
    We apply Eq.~\eqref{eq:lorenz_projection} to define the scalar-state set as 
    \begin{equation}
    \mathcal{C}_{\mathtt{Chaos}} 
    = 
    \left\{ \Phi_{\mathtt{Chaos}}(\xi_i) \;
    \middle| \; 
    \xi_i\in\left\{ 
    T_{\mathrm{tr}} 
    +
    n\Delta t \right\}_{n=0}^{S-1}
    \right\}. 
    \end{equation} 
    Given \(\mathbf{w}\), we compute \begin{equation}
    \xi^{*} = \arg\min_{\mathcal{C}_{\mathrm{Chaos}}} 
    \left\| 
    \Phi_{\mathrm{Lorenz}}(\xi) - \mathbf{w} \right\| . \label{eq:chaos_inverse_state} 
    \end{equation}
    This assigns \(\mathbf{w}\) to the nearest \(\xi^*\).

\end{itemize}

\textbf{Remark 2}. We can store a table or an explicit relationship between $j$ and $\xi_j$ such that only $j$ needs to be stored:
\[
    j \to \xi_j \to \Phi(\xi_j) \to \w,
\]
thereby $\w$ is compressed into the integer index, which is similar to the classical vector quantization. This can be viewed as a combination of dynamic system+quantization. For \texttt{SFC} systems, compression directly combines lattice quantization with a deterministic dynamical mapping: quantization defines a finite regular domain, while the dynamic system provides the ordering and indexing rule. For \texttt{LCG}, \texttt{PCG}, \texttt{QMC} sequences and chaotic systems, they directly  produce a finite state set, and \(\w\) is represented by the state which is nearest to it. Then, the associated index is stored.

Yet, quantization is not an essential step for realizing compression. Without quantization, we can find a point \(\xi\) such that $\Phi(\xi)$ can represent \(\w\); therefore, compression can also be realized. In conclusion, the essential compression mechanism is the low-dimensional scalar parameters to high-dimensional states.

\subsection{Engineering Techniques}
\label{encoding and reconstrction}
    Here, we introduce our engineering techniques to make the proposed model compression via dynamic systems adapt well in real-world model compression. We mainly deal with two issues: i) how to compress fast so that our method can handle large models with billions of parameters; ii) how to control the representation error for weights to avoid the performance collapse of the compressed model. 
    
    \textbf{Fast compression}: The major computational cost in the proposed framework arises from assigning \(\w\) to its nearest state \(\Phi(\xi)\) (Eq. \eqref{key_equation}), particularly for a large candidate set $\mathcal{C}$.
    To accelerate the computation, we adopt different assignment strategies for different generator families. For the two-dimensional \texttt{Lorenz} system, we construct a KD-tree over all $S$ states in $\mathcal{C}$ once and reuse it across layers. We perform nearest-neighbor queries with an average complexity of $O(\log(S))$ per query.
    
    For \(\mathtt{SFC}\) methods, because their states in $\mathcal{C}$ are integers, Eq.~\eqref{sfc_integer} directly quantizes \(\w\) to an integer lattice coordinate. 
    If we still build a KD-tree, this integer structure will be ignored and introduce additional construction cost.
    Thus, we directly use the state set $\mathcal{C}$ as a template, where each state is associated with a scalar index \(j\). 
    The template is generated once under a canonical local coordinate system and reused across different layers by translation and scaling, if necessary, to align with the hypercube $[c-\frac{B}{2},c+\frac{B}{2}]^D$.
    This reduces both assignment overhead during compression and lookup overhead during decompression, because the stored index can directly refer to the corresponding state.
    
    For \texttt{LCG}, \texttt{PCG} and \texttt{QMC}, we perform exact Euclidean nearest-state assignment over \(S\) candidate states. 
    The canonical candidate set is generated once and reused across layers through translation and scaling.
    To limit memory usage during nearest-state search, the $D$-dimensional grouped vectors are processed in batches, while the candidate states are further scanned in memory-bounded blocks when necessary.
    The minimum distance and corresponding index are updated across all blocks. Thus, the assignment remains exact, while the computational cost remains \(O(SD)\) per grouped vector.

\textbf{Error Control}:
For decompression, we need to restore the weight vector \(\tilde{\w}\) from index \(j\) to:
\[
    j \to \xi_j \to \Phi(\xi_j) \to \w \to \bar{\w} \to \tilde{\w},
\]
The error mainly occurs at $\Phi(\xi_j) \to \w$, and \(\w \to \bar{\w}\) may amplify the error. Since increasing $S$ can lower the error in $\Phi(\xi_j) \to \w$, we let $S$ increase quadratically until the error is acceptable.

For \(\w \to \bar{\w}\), we need to scale it into \([\,\mathbf{c}-\frac{B}{2},\mathbf{c}+\frac{B}{2}\,]^D\). Mathematically, let
\begin{equation}
    \rho_{\max}
    =
    \max_i
    \left\|
    \bar{\mathbf{w}}_i
    -
    \mathbf{c}
    \right\|_{\infty}.
    \label{eq:max_radius}
\end{equation}
Vectors inside
\([\,\mathbf{c}-\frac{B}{2},\mathbf{c}+\frac{B}{2}\,]^D\)
are assigned to the inner class \(k=0\). For the outer vectors, the radial interval
\((B/2,\rho_{\max}]\) is divided into \(K\) different domains, and \(k\) is the class number:
\begin{equation}
    \rho_k
    =
    \frac{B}{2}
    +
    k
    \frac{
    \rho_{\max}-\frac{B}{2}
    }{K},
    \qquad
    k=1,\ldots,K,
    \label{eq:radial_boundary}
\end{equation}
where $K$ is selected dynamically according to an \(\mathtt{MAE}\)-based loss threshold to control the error.
Thus, the closer $\bar{\w}_i$ require only mild scaling, while farther $\bar{\w}_i$ are scaled more strongly. 
Given $\bar{\w}_i$, it is assigned to the smallest class:
\begin{equation}
    \min
    \left\{
    k\in\{1,\ldots,K\}:
    \left\|
    \bar{\mathbf{w}}_i
    -
    \mathbf{c}
    \right\|_{\infty}
    \leq
    \rho_k
    \right\}.
    \label{eq:class_assignment}
\end{equation}
The corresponding scaling factor is
\begin{equation}
    a_{k}
    =
    \frac{B/2}{\rho_{k}} ,
    \label{eq:scaling_factor}
\end{equation}
and the scaled vector is
\begin{equation}
    \mathbf{w}_i
    =
    \mathbf{c}
    +
    a_{k}
    \left(
    \bar{\mathbf{w}}_i
    -
    \mathbf{c}
    \right).
    \label{eq:radial_scaling}
\end{equation} 

Furthermore, we can combine the class information with $j$ by updating \(j\) with the radial class label \(k\) as
\begin{equation}
    j'
    =
    j
    +
    k S ,
    \label{eq:combined_index}
\end{equation}
During decompression, $j'$ is parsed into \(k\) and \(j\) by the following:
\begin{equation}
    k
    =
    \left\lfloor
    j'/S
    \right\rfloor,
    \qquad
    j
    =
    j'
    \bmod
    S .
    \label{eq:index_parse}
\end{equation}
Thus, the required bit width is
\(
\left\lceil
\log_2\big((K+1)S\big)
\right\rceil .
\)

Furthermore, we set a mean-\(\mathtt{MAE}\) threshold $\eta$ for a tensor to control the error.
After encoding $\w$, we decompress it and compute the \(\mathtt{MAE}\) in place.
If the mean \(\mathtt{MAE}\) exceeds the threshold, the corresponding tensor $\w_i$ is stored directly without compression.

\section{Analysis Experiments}
\label{sec:analysis_experiments}

Before formally comparing dynamic system-based compression methods with other state-of-the-arts, here we first conduct analytical experiments to investigate how different hyper-parameters affect the compression--accuracy trade-off, with an emphasis on the grouping dimension and state scale. Then, we systematically compare the compression performance of different dynamic system-based compression methods using the Pareto curves. The analysis experiments can not only enrich the understanding of each method but also help choose representative configurations for them.

\textbf{Experimental Setup}.
\label{subsec:resnet_setup}
We evaluate \texttt{Hilbert}, \texttt{Morton/Z-order}, \texttt{Peano}, \texttt{Snake}, \texttt{Lorenz}, \texttt{LCG}, \texttt{PCG}, and \texttt{QMC/Halton}. The dimension and scale settings are summarized in Table~\ref{tab:evaluated_settings}. Across the eight methods, we evaluate 254 configurations. We use the standard torchvision ResNet-18 pretrained on ImageNet-1K as the backbone.\footnote{
We use the torchvision \texttt{ResNet18\_Weights.IMAGENET1K\_V1} checkpoint, whose official weight file is \texttt{resnet18-f37072fd.pth}: \url{https://download.pytorch.org/models/resnet18-f37072fd.pth}.
}
The original checkpoint occupies 44.7 MB and achieves Top-1 accuracy \(69.76\%\) and Top-5 accuracy \(89.07\%\) on the
50,000-image ILSVRC 2012 validation set ~\cite{russakovsky2015imagenet}.
We follow the standard evaluation protocol by resizing each image to 256 pixels, applying a \(224\times224\) center crop, and using the  normalization recommended here~\cite{torchvision2026resnet18}.
The reconstructed models are evaluated without any test-time augmentation.
All experiments are conducted on servers equipped with NVIDIA L40 and A800 GPUs.

All methods are evaluated under the same tensor-level compression protocol. 
Each tensor is either stored directly or compressed by the corresponding dynamic system. As common practice \cite{lin2024awq}, the buffers, biases, very small tensors, and tensors whose reconstruction error exceeds the error threshold will be considered as outliers and therefore stored directly. For tensors compressed, we save their state indices and the associated metadata for decompression.

\begin{table}[t]
\renewcommand{\arraystretch}{1.4}
\centering
\caption{Evaluated state-generation settings. 
For grid-based mappings, \(2^k\) (or $3^k$ for Peano) denotes the per-axis resolution, yielding $S=2^{kD}$ (or $S=3^{kD}$) system states. 
For trajectory- and sequence-based mappings, $S$ denotes the system states number.}
\label{tab:evaluated_settings}

\renewcommand{\arraystretch}{1.10}
\setlength{\tabcolsep}{3.5pt}
\footnotesize

\begin{tabularx}{\columnwidth}{
>{\raggedright\arraybackslash}p{1.65cm}
>{\centering\arraybackslash}p{0.75cm}
>{\raggedright\arraybackslash}X
}
\hline
\rowcolor{gray!16}
\textbf{Method}
&
\textbf{\(D\)}
&
\textbf{Scale setting}
\\
\hline

\multicolumn{3}{>{\columncolor{lightblue!100}}l}{
\textbf{Space-filling curves and traversal mappings}} \\
\rowcolor{lightblue!60}
\texttt{Hilbert}
&
\(2\)--\(3\)
&
\(S=2^{k\cdot D},\; k\in\{4,5,6,7,8,9,10\}\)
\\
\rowcolor{lightblue!60}
&
\(4\)--\(6\)
&
\(S=2^{k\cdot D},\; k\in\{3,4,5,6,7,8,9,10\}\)
\\
\rowcolor{lightblue!60}
\texttt{Snake}
&
\(2\)--\(3\)
&
\(S=2^{k\cdot D},\; k\in\{4,5,6,7,8,9,10\}\)
\\
\rowcolor{lightblue!60}
&
\(4\)--\(6\)
&
\(S=2^{k\cdot D},\; k\in\{3,4,5,6,7,8,9,10\}\)
\\
\rowcolor{lightblue!60}
\texttt{Z-order}
&
\(2\)--\(3\)
&
\(S=2^{k\cdot D},\; k\in\{4,5,6,7,8,9,10\}\)
\\
\rowcolor{lightblue!60}
&
\(4\)--\(6\)
&
\(S=2^{k\cdot D},\; k\in\{3,4,5,6,7,8,9,10\}\)
\\
\rowcolor{lightblue!60}
\texttt{Peano}
&
\(2\)--\(6\)
&
\(S=3^{k\cdot D},\; k\in\{3,4,5,6\}\)
\\

\hline
\multicolumn{3}{>{\columncolor{blue!10}}l}{\textbf{Chaotic trajectory-based mapping}} \\
\rowcolor{blue!6}
\texttt{Lorenz}
&
\(2\)
&
\(S=2^{18}, 2^{20},2^{22},2^{24}\)
\\

\hline
\multicolumn{3}{>{\columncolor{green!12}}l}{\textbf{Pseudo-random sequence mappings}} \\
\rowcolor{green!6}
\texttt{LCG}
&
\(2\)
&
\(S=2^k,\; k\in\{6,8,10,12,14,16,18\}\)
\\
\rowcolor{green!6}
&
\(3\)
&
\(S=2^k,\; k\in\{8,10,12,14,16,18,20,22\}\)
\\
\rowcolor{green!6}
&
\(4\)
&
\(S=2^k,\; k\in\{12,14,16,18,20,22,24,26\}\)
\\
\rowcolor{green!6}
&
\(5,6\)
&
\(S=2^k,\; k\in\{12,14,16,18,20,22,24,26,28\}\)
\\

\rowcolor{green!6}
\texttt{PCG}
&
\(2\)
&
\(S=2^k,\; k\in\{6,8,10,12,14,16,18\}\)
\\
\rowcolor{green!6}
&
\(3\)
&
\(S=2^k,\; k\in\{8,10,12,14,16,18,20,22\}\)
\\
\rowcolor{green!6}
&
\(4\)
&
\(S=2^k,\; k\in\{12,14,16,18,20,22,24,26\}\)
\\
\rowcolor{green!6}
&
\(5,6\)
&
\(S=2^k,\; k\in\{12,14,16,18,20,22,24,26,28\}\)
\\
\hline
\multicolumn{3}{>{\columncolor{yellow!16}}l}{\textbf{Low-discrepancy sequence mapping}} \\
\rowcolor{yellow!6}
\texttt{QMC}
&
\(2\)
&
\(S=2^k,\; k\in\{6,8,10,12,14,16,18,20\}\)
\\
\rowcolor{yellow!6}
&
\(3\)
&
\(S=2^k,\; k\in\{8,10,12,14,16,18,20,22\}\)
\\
\rowcolor{yellow!6}
&
\(4\)
&
\(S=2^k,\; k\in\{12,14,16,18,20,22,24,26\}\)
\\
\rowcolor{yellow!6}
&
\(5,6\)
&
\(S=2^k,\; k\in\{20,22,24,26,28\}\)
\\

\hline
\end{tabularx}
\vspace{-0.72cm}
\end{table}

\textbf{Overall Compression--Accuracy Behavior}.
\label{subsec:resnet_overall_behavior}
The measured compression ratio is governed by two coupled
factors: the dimensionality ($D$) which is the number of weights to compress at once, and the scale ($S$). We visualize the joint effects of $D$ and $S$ on the measured compression ratio and Top-1 accuracy drop in Figures ~\ref{fig:cr_3d_all} and~\ref{fig:top1drop_3d_all}.

\begin{figure*}[t]
\centering
\includegraphics[width=\textwidth]{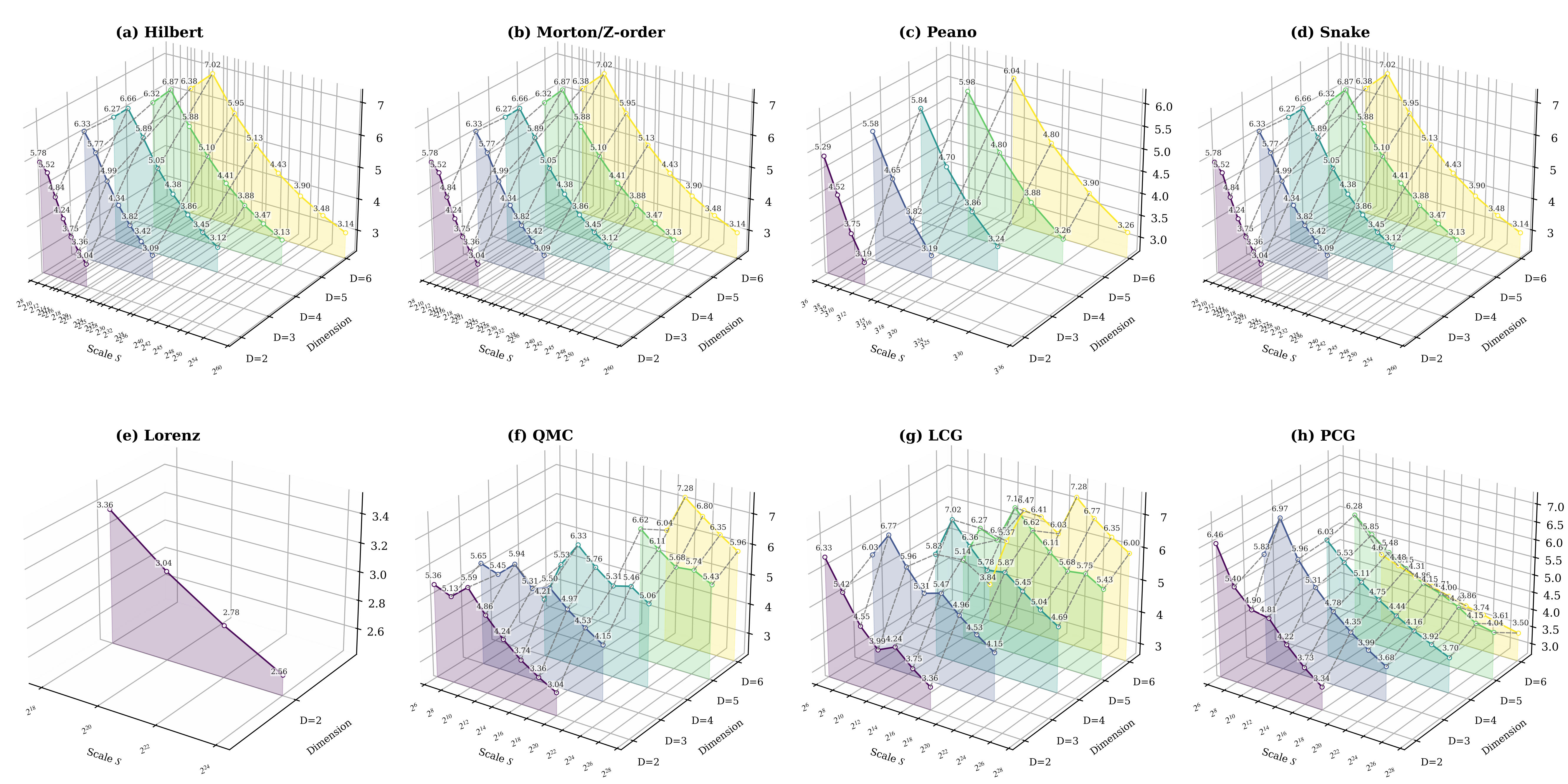}
\caption{
Compression-ratio landscapes of the eight evaluated methods on ResNet-18.
Each panel corresponds to one method, where the horizontal axis denotes the system scale, the depth axis denotes the grouping dimension \(D\), and the vertical axis reports the measured compression ratio.
For each fixed dimension, the curve and translucent curtain connect the evaluated scale configurations, while the numerical labels indicate the corresponding compression ratios.
Higher values represent stronger compression.
Lorenz is evaluated only in the two-dimensional setting because its trajectory states are projected onto a two-dimensional representation.
}
\label{fig:cr_3d_all}
\vspace{-0.5cm}
\end{figure*}

\begin{figure*}[t]
\centering
\includegraphics[width=\textwidth]{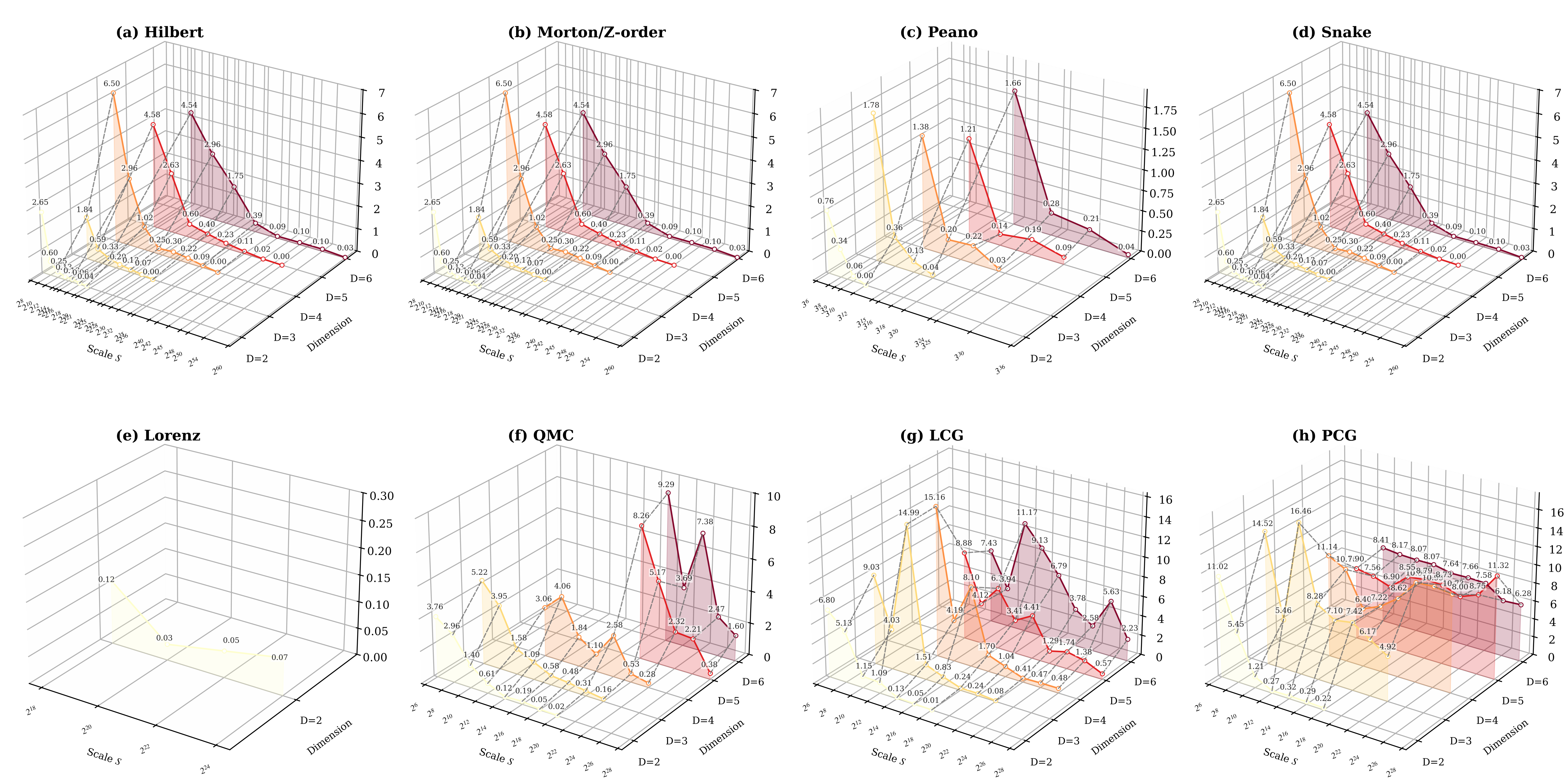}
\caption{
Top-1 accuracy-drop landscapes of the eight evaluated methods on ResNet-18.
The vertical axis reports the Top-1 accuracy drop relative to the original model in percentage points.
Lower values indicate better accuracy preservation.
}
\label{fig:top1drop_3d_all}
\vspace{-0.5cm}
\end{figure*}

\textit{SFC.}
A notable observation from Figures ~\ref{fig:cr_3d_all} and ~\ref{fig:top1drop_3d_all} is that \texttt{Hilbert}, \texttt{Morton/Z-order} and \texttt{Snake} yield identical compression ratios and accuracies under all matched configurations.
For a grouping dimension \(D\) and a per-axis resolution \(m=2^k\), the three methods traverse the same state grid:
\begin{equation}
        \mathcal{C}_{m,D}=\{\Phi_{T}(\xi_i)\}_{i=0}^{m^D-1}
        \quad
        T\in\{\mathrm{Hilbert},\mathrm{Morton},\mathrm{Snake}\}.
\end{equation}
They differ only in the permutation used to assign scalar indices
to the grid points.
Because exact assignment selects the same reconstruction point from \(\mathcal{C}_{m,D}\), the reconstructed tensors and threshold-based fallback decisions remain unchanged.
Although the stored index values differ across traversal orders, they are stored in the same bit width. The resulting compression--accuracy behavior is thus determined by the grid geometry and resolution rather than by traversal order.

At \(D=2,3\), the compression performance shows a consistent and stable trend.
As \(S\) increases, both the loss of Top-1 accuracy and the compression ratio decrease, the latter decrease resulting from the enlarged index width. 
At \(D=4,5,6\), their compression ratios initially increase with per-axis resolution from \(2^3\) to \(2^4\), while the loss of accuracy decreases.
From the per-axis resolution of \(2^4\) onward, the compression performance becomes consistent and stable. The compression ratio and accuracy loss decrease with increasing \(S\).
\(D=4\) achieves the lowest Top-1 accuracy drop of \(0.25\%\) while keeping the compression ratio above \(5\times\).
\(D=5\) achieves the highest compression ratio of \(5.88\times\) constrained by the Top-1 accuracy drop below \(1\%\).

For a fixed \(D\), increasing \(S\) may expand the compressed tensor set, making previously uncompressed tensors eligible for compression.
This reduces fallback storage and
thus can improve the compression ratio. 
Therefore, the compression ratio increases when the reduction in fallback storage outweighs the larger
index cost; otherwise, the compression ratio decreases. 
Newly-compressed tensors may introduce additional error and thus can decrease the accuracy. 
Therefore, the accuracy decreases when the newly introduced error outweighs the overall error improvement gained from larger S; otherwise, the accuracy increases.

Furthermore, the \texttt{Peano} curve follows the same permutation invariance on an identical grid.
In our experiments, however, the \texttt{Peano} curve uses ternary resolutions \(m=3^k\), whereas the \texttt{Hilbert}, \texttt{Morton/Z-order}, and \texttt{Snake} curves use binary resolutions \(m=2^k\).
Both their evaluated grid points and scale do not match exactly. 
The observed differences therefore arise from their different resolution schedules rather than the traversal itself.

When the per-axis resolution is \(2^6\), for the \texttt{Hilbert}, \texttt{snake} and \texttt{Zorder} curves, all dimensions \(D=2\)--\(6\) achieve compression ratios
between \(4.84\)--\(5.13\times\) with a Top-1 drop at most
\(0.40\%\).
This range provides a favorable balance between compression efficiency and accuracy preservation.
For \texttt{Peano}, per-axis resolution of \(3^4\) achieves compression between \(4.52\)--\(4.80\times\) with Top-1 accuracy drops between \(0.14\)--\(0.36\%\).

\textit{Quasi-Monte Carlo}. This method exhibits more irregular and less monotonic compression--accuracy surfaces than the SFC-based methods.

At \(D=2,3\), increasing \(S\) generally prevents the Top-1 accuracy from dropping too much and reduces the compression ratio, resulting in a stable accuracy--compression trade-off.
For example, at \(D=3\), increasing \(S\) from \(2^8\) to \(2^{22}\) reduces the loss of Top-1 accuracy from \(5.22\%\) to \(0.16\%\), with the compression ratio decreasing from \(5.65\times\) to \(4.15\times\). 
\(D=3\) achieves the highest compression ratio of \(5.50\times\) at Top-1 accuracy drop below \(1\%\).

At \(D=4, 6\), the accuracy degradation exhibits larger fluctuations as \(S\) increases, indicating that performance depends not only on the state count, but also on the effective coverage of states and the tensors satisfying the compression error threshold.
\(D=4\) achieves the lowest Top-1 accuracy drop of \(0.28\%\) with the compression ratio above \(5\times\).

At \(D=5\), increasing \(S\) from \(2^{20}\) to \(2^{28}\) reduces the Top-1 accuracy drop from \(8.26\%\) to \(0.38\%\),
while the compression ratio decreases from \(6.62\times\) to \(5.43\times\) because of the larger bit width for indices.

\textit{Congruential and pseudo-random generators.}
\texttt{LCG} exhibits a more irregular response than \texttt{QMC}.
At \(D=2\), increasing \(S\) produces a stable compression performance.
At \(S=2^6\)--\(2^{18}\), the loss of Top-1 accuracy decreases further from \(6.80\%\) to \(0.01\%\), while the compression ratio also gradually decreases from \(6.33\times\) to \(3.36\times\).

At higher dimensions, compression performance fluctuations appear. 
Taking \(D=4\) as an example, as \(S\) increases from \(2^{12}\) to \(2^{18}\), the compression ratio changes from \(5.83\times\) to \(7.02\times\), \(6.36\times\) and \(5.78\times\), while the corresponding accuracy drops are \(15.16\%\), \(4.19\%\), \(8.10\%\) and \(1.70\%\).
From \(2^{12}\) to \(2^{14}\), the reduction in fallback storage outweighs the increase in index cost, resulting in both a higher compression ratio and better accuracy.
At \(S=2^{16}\), the accuracy drop temporarily increases, suggesting that newly-compressed tensors introduce perturbations.
At \(S=2^{18}\), accuracy preservation improves substantially, although the wider indices continue to reduce the compression ratio.

\texttt{LCG} attains its smallest Top-1 drop of \(0.41\%\) at \(D=4\), while keeping the compression ratio above \(5\times\).
Its highest compression ratio is \(5.47\times\) obtained at \(D=3\) with the Top-1 accuracy drop below \(1\%\).

\texttt{PCG} is more sensitive to the grouping dimension.
At \(D=2\), increasing \(S\) from \(2^{10}\) to \(2^{12}\) reduces the loss of Top-1 accuracy from \(1.21\%\) to \(0.27\%\), while the compression ratio decreases slightly from \(4.90\times\) to \(4.81\times\).
For \(S=2^{12}\)--\(2^{18}\), the Top-1 drop remains within \(0.22\)--\(0.32\%\). The compression ratio decreases from \(4.81\times\) to \(3.34\times\).
However, this favorable low-dimensional behavior does not extend to higher dimensions, where the \texttt{PCG} exhibits substantially larger accuracy degradation. \texttt{PCG} achieves the  Top-1 accuracy drop under \(1\%\) only at \(D=2\).

\textit{Lorenz} is evaluated only at \(D=2\), because the trajectory states are projected onto a two-dimensional representation.
Increasing \(S\) from \(2^{18}\) to \(2^{24}\) reduces the
compression ratio from \(3.36\times\) to \(2.56\times\), while the
Top-1 drop remains within \(0.03\)--\(0.12\%\).
Accuracy preservation is already close to saturation at
\(S=2^{18}\), and further extending the trajectory provides little additional accuracy benefit while constantly increasing the index cost.

These results suggest that DS-based compression is fundamentally a state-space approximation problem: a dynamic system defines the geometry of a finite representation space, while the compression quality depends on how effectively this state space covers the structure of model weights.
This explains why different systems with the same number of states can exhibit substantially different compression behaviors, while systems with different indexing rules may produce identical results, such as the \texttt{Hilbert}, \texttt{Morton/Z-order}, and \texttt{Snake} compression performance.
Therefore, effective dynamical representations should prioritize coverage of grouped-weight distributions under a constrained storage budget, rather than simply increasing trajectory complexity or optimizing index traversal.

\begin{table}[t]
\renewcommand{\arraystretch}{1.2}
\centering
\caption{
Representative configurations selected from the method-wise Pareto frontiers on ResNet-18.
CR denotes the measured compression ratio, and \(\Delta\) denotes the accuracy change relative to the original model, which achieves \(69.76\%\) Top-1 and \(89.07\%\) Top-5 accuracy.
}
\label{tab:main_results}
\setlength{\tabcolsep}{4pt}
\begin{tabular}{lcc >
{\columncolor{lightblue!80}}c >{\columncolor{lightyellow!80}}c >{\columncolor{lightyellow!80}}c}
\hline
\textbf{Method}& \textbf{D} & \textbf{S}  & \textbf{CR} & \textbf{Top-1 Acc.} ($\Delta$)\% & \textbf{Top-5 Acc.} ($\Delta$)\%\\
\hline
\texttt{Hilbert}  & 3 & \(2^{15}\) & \underline{\(\textbf{5.77}\)}$\times$ & 69.17 (-0.59) & 88.67 (-0.40)\\
\texttt{Peano} & 5 & \(3^{20}\) & \(4.80\times\) & 69.62 (-0.14) & 88.91 (-0.16) \\
\texttt{Lorenz} & 2 &  \(2^{18}\) & \(3.36\times\) & \underline{\textbf{69.64 (-0.12)}} & \underline{\(\textbf{89.07 (0.00)}\)} \\
\texttt{LCG} & 3 & \(2^{18}\) & \(4.96\times\) & 69.52 (-0.24) & 88.99 (-0.08)\\
\texttt{PCG} & 2 & \(2^{12}\) & \(4.81\times\) & 69.49 (-0.27) & 88.86 (-0.21) \\
\texttt{QMC} & 3 & \(2^{16}\) &\(5.50\times\)& 69.18 (-0.58) & 88.72 (-0.35)\\
\hline
\end{tabular}
\vspace{-0.54cm}
\end{table}

We jointly analyze the performance of the compression ratio and the accuracy drop using global and method-wise Pareto frontiers, which are provided in Supplementary Figs.~S1 and~S2.
Based on the method-wise frontiers, we select one representative configuration for each method and summarize the configurations in Table~\ref{tab:main_results}.
Excluding \texttt{Lorenz}, representative configurations achieve the compression ratio between \(4.80\)--\(5.77\times\) with the accuracy drop between \(0.14\)--\(0.59\%\).
Lorenz achieves a \(3.36\times\) compression ratio with a \(0.12\%\) accuracy drop. 

\section{Comparative Experiments}
\label{sec:comparative_experiments}
 
Through analysis experiments, the characteristics of different dynamic systems and the hyperparameter configurations are probed. For large language models, we compare the proposed dynamic system-based compression with other state-of-the-art compression methods. Experiments show that DS compression methods can deliver competitive performance. 

\textbf{Experimental Setup}.
All six representative configurations in Table~\ref{tab:main_results} are cast for six dynamic systems. 
The benchmark models are the FP16 checkpoints of Qwen2.5-1.5B and Qwen1.5-7B ~\cite{qwen2024qwen25,bai2023qwen}, which account for 3.09 and 15.44 GB in memory, respectively.
We compress all eligible parameter tensors using the proposed methods, while tensors whose decompression errors exceed the prescribed threshold are directly retained in the original form.

Compression quality is evaluated from two complementary perspectives:  language modeling quality and downstream-task performance.
We evaluate perplexity (PPL) on all complete non-overlapping 2048-token blocks of the WikiText-2 raw test split~\cite{merity2017pointer}.
Downstream performance is evaluated on 16 tasks covering knowledge and question answering (SciQ, BoolQ, and MMLU)~\cite{welbl2017sciq,clark2019boolq,hendrycks2021mmlu}, commonsense and causal reasoning
(WinoGrande, HellaSwag, PIQA, and COPA)~\cite{sakaguchi2020winogrande,zellers2019hellaswag,bisk2020piqa,gordon2012copa}, scientific and mathematical
reasoning (ARC-Easy, ARC-Challenge, and GSM8K)~\cite{clark2018arc,cobbe2021gsm8k}, and natural language inference, reading comprehension, and
linguistic analysis (CB, MultiRC, ReCoRD, RTE, WiC, and WSC) from SuperGLUE~\cite{wang2019superglue}.

\textbf{Comparison with Existing Compression Methods}.
Table~\ref{tab:llm_scaling_results} summarizes the performance of six DS-based compression methods on Qwen2.5-1.5B and Qwen1.5-7B.
Across both model scales, all configurations exhibit distinct compression--quality operating points while maintaining perplexity and downstream performance close to the corresponding uncompressed models.

\begin{table}[t]
\renewcommand{\arraystretch}{1.2}
\centering
\caption{
Compression results on Qwen2.5-1.5B and Qwen1.5-7B.
CR denotes the compression ratio.
PPL is evaluated on WikiText-2, and Avg. Score is averaged over 16 downstream tasks.
Values in parentheses denote changes relative to the corresponding uncompressed model.
}
\label{tab:llm_scaling_results}
\setlength{\tabcolsep}{4pt}
\renewcommand{\arraystretch}{1.08}
\begin{tabular}{l l >{\columncolor{lightblue!80}}c >{\columncolor{lightgreen!80}}c >{\columncolor{lightyellow!80}}c}
\hline
\textbf{Model} 
& \textbf{Method} 
& \textbf{CR} 
& \textbf{PPL}($\Delta$)
& \textbf{Avg. Score} ($\Delta$)\%\\
\hline

\multirow{7}{*}{\rotatebox{90}{Qwen2.5-1.5B}}
& Original 
& -- 
& 9.26
& 67.31 \\

& \texttt{Hilbert}
& \(\underline{\textbf{2.67}}\times\) 
& 9.61 (+0.35) 
& 66.52 (-0.79) \\

& \texttt{Peano}
& 2.37$\times$ 
& 9.55 (+0.29) 
& \(\underline{\textbf{67.70 (+0.39) }}\) \\

& \texttt{Lorenz}
& 1.68$\times$
& \(\underline{\textbf{9.29 (+0.03)}}\) 
& 67.31 (+0.00) \\

& \texttt{QMC}
& 2.56$\times$
& 9.75 (+0.49) 
& 66.44 (-0.87) \\

& \texttt{LCG}
& 2.36$\times$
& 9.62 (+0.36) 
& 66.79 (-0.52) \\

& \texttt{PCG}
& 2.33$\times$
& 9.66 (+0.40) 
& 66.24 (-1.07) \\
\hline

\multirow{7}{*}{\rotatebox{90}{Qwen1.5-7B}}
& Original  
& -- 
& 11.30
& 65.49 \\

& \texttt{Hilbert}
& \(\underline{\textbf{3.00}}\times\) 
& 11.81 (+0.51)
& \(\underline{\textbf{66.12 (+0.63)}}\) \\

& \texttt{Peano}
& 2.42$\times$ 
& \(\underline{\textbf{11.30 (+0.00)}}\)
& 64.99 (-0.50) \\

& \texttt{Lorenz}
& 1.68$\times$
& 11.32 (+0.02)
& 65.32 (-0.17) \\

& \texttt{QMC}
& 2.82$\times$
& 11.96 (+0.66) 
& 65.26 (-0.23) \\

& \texttt{LCG}
& 2.52$\times$
& 11.52 (+0.22) 
& 65.00 (-0.49) \\

& \texttt{PCG}
& 2.46$\times$
& 11.63 (+0.33) 
& 65.20 (-0.29) \\
\hline
\end{tabular}
\vspace{-0.2cm}
\end{table}

\begin{figure}[t]
\centering
\includegraphics[width=\linewidth]{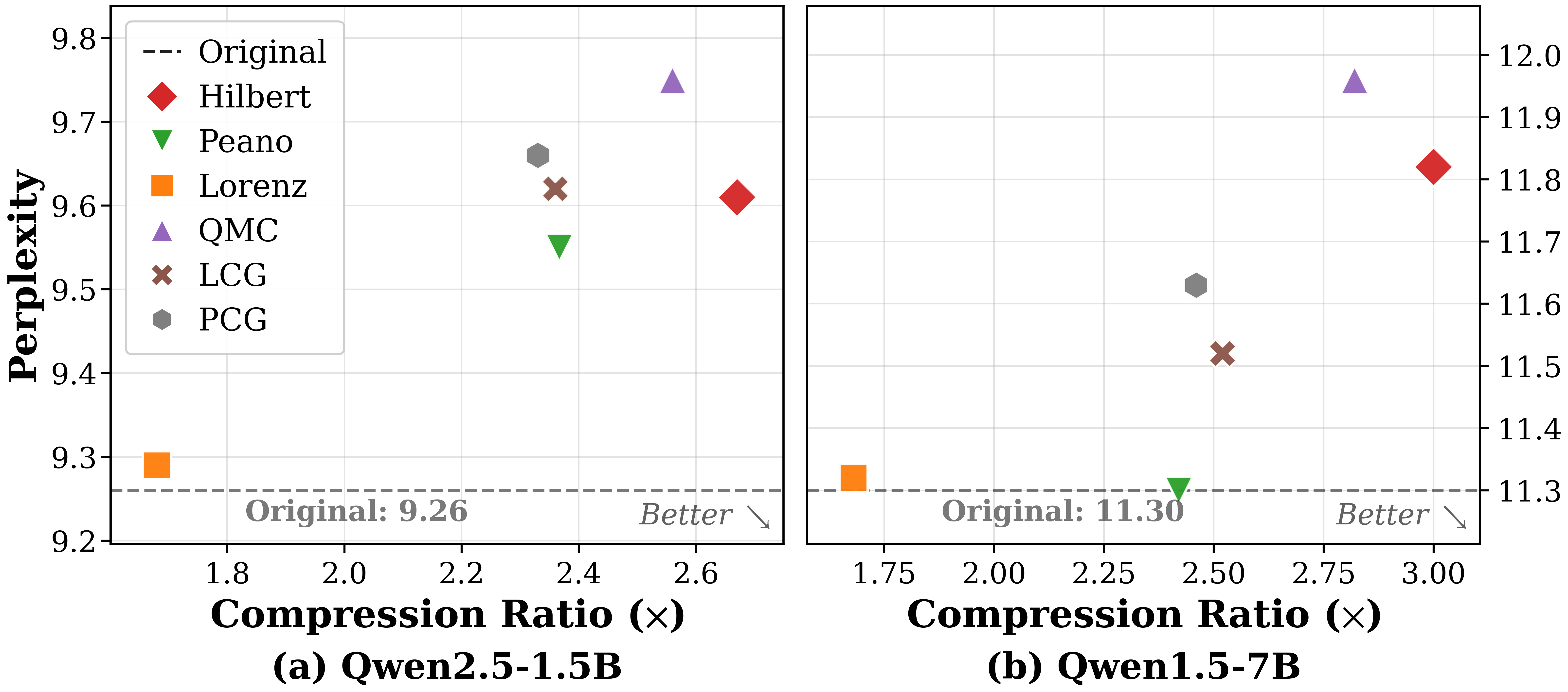}
\caption{
Compression--PPL trade-offs on Qwen2.5-1.5B and Qwen1.5-7B.
Dashed lines denote the PPL of the corresponding uncompressed models.
Higher compression ratios and lower perplexity indicate better trade-offs.
}
\label{fig:llm_crppl}
\vspace{-0.5cm}
\end{figure}

Figure~\ref{fig:llm_crppl} further visualizes the compression-PPL trade-offs, showing distinct operating points across the six dynamic system configurations.
$\mathtt{Lorenz}$ preserves near-baseline perplexity, with only 0.03 and 0.02 PPL increases on Qwen2.5-1.5B and Qwen1.5-7B, respectively, while maintaining a conservative compression ratio of $1.68\times$.
In contrast, $\mathtt{Hilbert}$ achieves the highest compression ratios among the evaluated configurations, reaching $2.67\times$ and $3.00\times$ on the two models with PPL increases of $0.35$ and $0.51$.
Other configurations occupy intermediate positions, providing different compression-quality operating points. In particular, \texttt{QMC} reaches relatively high compression ratios of $2.56\times$ and $2.82\times$, but is accompanied by the largest PPL increases at $0.49$ and $0.66$, respectively. This consistent trend indicates that the proposed DS method scales beyond convolutional networks to billion-parameter language models.
Furthermore, the average downstream scores also remain close to those of the corresponding uncompressed models, indicating that the obtained trade-offs are not subjected to systematic degradation in
aggregate downstream tasks.

Among six configurations evaluated above, \texttt{Hilbert} achieves the highest compression ratio on two models.
We therefore select it for comparison with existing LLM compression methods. 
We compare against the irrational-winding reference (IR)
\cite{fan2026hypercompression},
weight-only configurations of conventional low-bit post-training quantization methods including
ZeroQuant~\cite{yao2022zeroquant}, SQuant~\cite{guo2022squant}, and vanilla PTQ,
as well as representative weight-only methods including
GPTQ~\cite{frantar2023gptq},
AWQ~\cite{lin2024awq},
SpQR~\cite{dettmers2024spqr}, and QuIP~\cite{chee2023quip}.
All compared configurations compress model weights while keeping activations and KV caches at FP16 precision.
For a consistent comparison, all methods are evaluated on the same model checkpoints and the same evaluation protocols. Compression ratios are computed using the same storage accounting rule.

\begin{figure}[t]
\centering
\includegraphics[width=\linewidth]{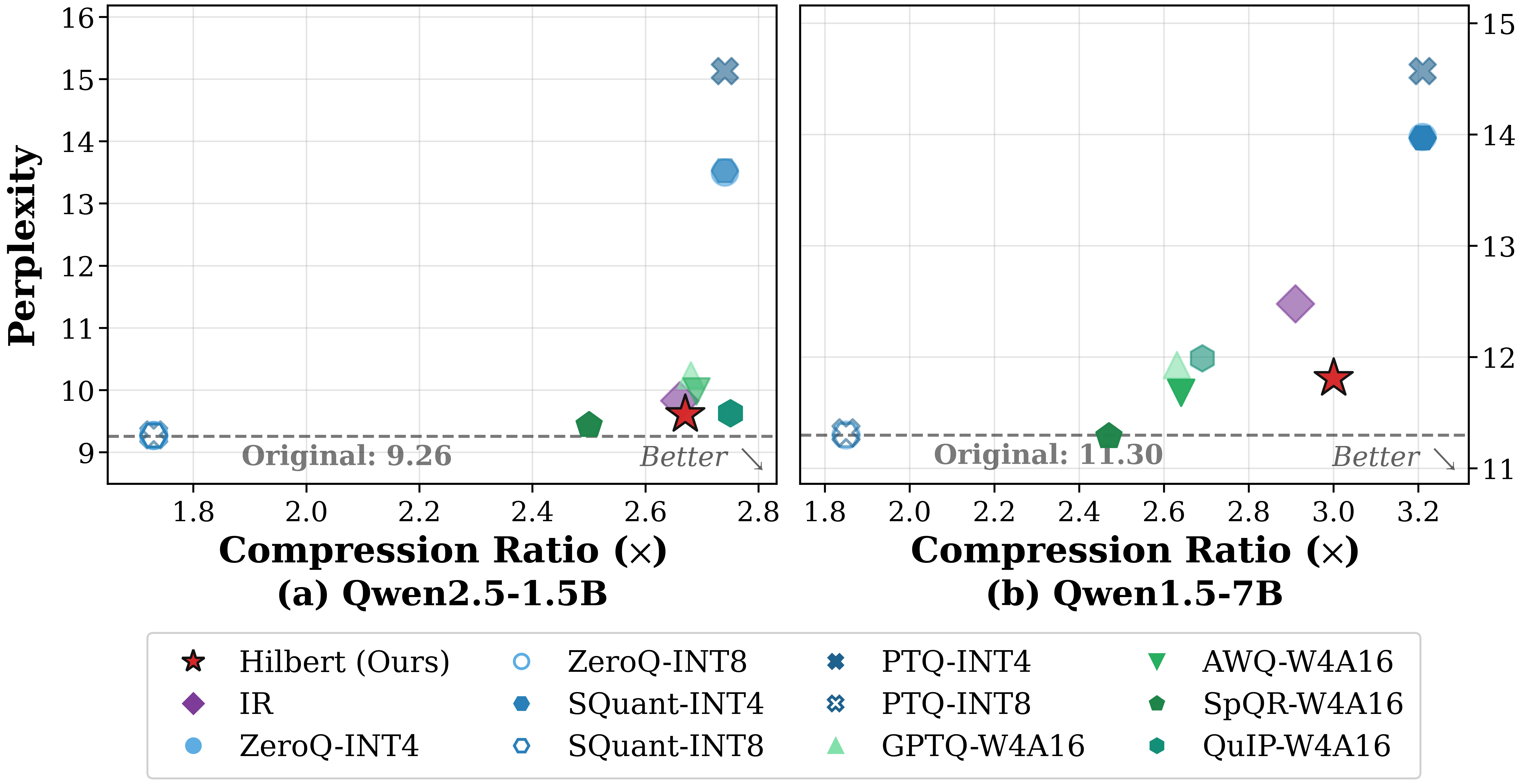}
\caption{
Compression--perplexity trade-offs of weight-only compression methods on Qwen2.5-1.5B and Qwen1.5-7B.
Higher CR and lower PPL indicate better trade-offs.
}
\label{fig:llm_compression_pareto}
\vspace{-0.5cm}
\end{figure}

\begin{table*}[t]
\centering
\caption{
Compression and downstream-task comparison on Qwen2.5-1.5B.
CR denotes the compression ratio, PPL is evaluated on WikiText-2. 
Task scores are reported in percentage (\%).
The best downstream task score is highlighted in bold.
Avg. Score is averaged over the 16 downstream tasks.
ZQ, SQ, and PTQ denote ZeroQuant, SQuant, and vanilla PTQ, respectively.
}
\label{tab:qwen15b_downstream_comparison}

\footnotesize
\setlength{\tabcolsep}{3.0pt}
\renewcommand{\arraystretch}{1.05}

\begin{tabular}{l>{\columncolor{lightgray!20}}c>{\columncolor{lightgray!20}}c >{\columncolor{lightblue!80}}c >
{\columncolor{lightyellow!80}}c>{\columncolor{lightyellow!80}}c>{\columncolor{lightyellow!80}}c>{\columncolor{lightyellow!80}}c>{\columncolor{lightyellow!80}}c>{\columncolor{lightyellow!80}}c>{\columncolor{lightgreen!80}}c>
{\columncolor{lightgreen!80}}c>{\columncolor{lightgreen!80}}c>{\columncolor{lightgreen!80}}c}

\toprule

& \multicolumn{1}{c}{Reference}
& \multicolumn{1}{c}{Dyn-Sys}
& \multicolumn{1}{c}{Previous}
& \multicolumn{6}{c}{Conventional PTQ}
& \multicolumn{4}{c}{Advanced Weight-only}
\\
\cmidrule(lr){2-2}
\cmidrule(lr){3-3}
\cmidrule(lr){4-4}
\cmidrule(lr){5-10}
\cmidrule(lr){11-14}

Metric / Task
& Original
& \texttt{Hilbert}
& IR
& ZQ4
& ZQ8
& SQ4
& SQ8
& PTQ4
& PTQ8
& GPTQ
& AWQ
& SpQR
& QuIP
\\

\hline
\textbf{CR}
& --
& 2.67$\times$
& 2.67$\times$
& 2.74$\times$
& 1.73$\times$
& 2.74$\times$
& 1.73$\times$
& 2.74$\times$
& 1.73$\times$
& 2.68$\times$
& 2.69$\times$
& 2.50$\times$
& 2.75$\times$
\\

\textbf{PPL}
& 9.26
& 9.61
& 9.84
& 13.50
& 9.27
& 13.53
& 9.27
& 15.14
& 9.28
& 10.24
& 9.98
& 9.45
& 9.63
\\

\textbf{($\Delta$ PPL)}
& --
& +0.35
& +0.58
& +4.24
& +0.01
& +4.27
& +0.01
& +5.88
& +0.02
& +0.98
& +0.72
& +0.19
& +0.37
\\

SciQ
& 93.30
& \textbf{94.20}
& 93.70
& 89.60
& 93.40
& 89.60
& 93.40
& 86.30
& 93.30
& 93.70
& 93.80
& 93.40
& 93.40
\\

WinoGrande
& 63.38
& 62.19
& \textbf{64.40}
& 58.17
& 63.93
& 59.75
& 63.69
& 59.27
& 63.38
& 62.51
& 63.85
& 63.54
& 64.25
\\

ARC-Easy
& 71.51
& 72.85
& 72.60
& 66.20
& 71.04
& 67.38
& 71.34
& 58.71
& 71.55
& \textbf{73.36}
& 72.64
& 71.55
& 73.19
\\

ARC-Challenge
& 54.78
& 52.39
& 53.33
& 45.14
& \textbf{54.86}
& 45.05
& 54.61
& 43.52
& 54.61
& 51.96
& 51.96
& 54.18
& 51.88
\\

HellaSwag
& 67.95
& 67.40
& 66.63
& 61.27
& 67.92
& 61.39
& \textbf{67.95}
& 59.85
& 67.86
& 66.72
& 66.29
& 67.42
& 66.39
\\

BoolQ
& 78.84
& 78.75
& 77.86
& 65.69
& 78.81
& 64.98
& 78.75
& 43.09
& \textbf{78.96}
& 76.27
& 78.10
& 77.16
& 77.71
\\

MMLU
& 60.95
& 60.06
& 59.74
& 50.81
& \textbf{61.00}
& 50.22
& 60.85
& 49.91
& 60.90
& 57.39
& 58.28
& 60.20
& 59.07
\\

GSM8K
& 62.62
& 59.14
& 54.21
& 16.38
& 62.02
& 15.69
& 61.56
& 12.59
& \textbf{62.93}
& 46.55
& 53.37
& 59.06
& 51.71
\\

PIQA
& 75.95
& \textbf{76.33}
& 75.79
& 74.10
& 76.01
& 74.10
& 75.90
& 72.63
& 76.12
& 75.41
& 75.68
& 75.46
& 74.70
\\

CB
& 71.43
& 64.29
& 67.86
& 50.00
& \textbf{73.21}
& 46.43
& \textbf{73.21}
& 41.07
& 67.86
& 64.29
& 60.71
& 58.93
& 46.43
\\

COPA
& 83.00
& \textbf{85.00}
& 84.00
& 81.00
& 83.00
& 79.00
& 83.00
& 77.00
& 83.00
& 83.00
& 83.00
& 82.00
& 82.00
\\

MultiRC
& 28.55
& 30.38
& 42.06
& 40.18
& 28.24
& 41.30
& 28.28
& 44.66
& 29.08
& \textbf{48.41}
& 40.45
& 41.63
& 35.25
\\

ReCoRD
& 84.83
& 83.90
& 84.06
& 78.95
& 84.75
& 78.55
& 84.78
& 75.39
& \textbf{84.86}
& 83.27
& 83.47
& 84.38
& 84.29
\\

RTE
& 70.04
& 64.98
& 67.51
& 54.51
& 70.76
& 53.79
& 69.68
& 59.57
& 68.23
& \textbf{71.12}
& 68.59
& 63.90
& 64.62
\\

WiC
& 53.13
& 53.76
& \textbf{58.78}
& 50.78
& 53.13
& 51.88
& 53.13
& 51.88
& 53.45
& 57.37
& 51.88
& 57.05
& 54.86
\\

WSC
& 56.73
& 58.65
& 53.85
& \textbf{63.46}
& 58.65
& \textbf{63.46}
& 57.69
& \textbf{63.46}
& 57.69
& 36.54
& 47.12
& 56.73
& 57.69
\\
\textbf{Avg. Score}
& 67.31
& \shortstack{66.52}
& \shortstack{67.27}
& \shortstack{59.14}
& \shortstack{67.55}
& \shortstack{58.91}
& \shortstack{67.36}
& \shortstack{56.18}
& \shortstack{67.11}
& \shortstack{65.49}
& \shortstack{65.57}
& \shortstack{66.66}
& \shortstack{64.84}
\\
\textbf{($\Delta$ Avg)}
& --
& \shortstack{(-0.79)}
& \shortstack{(-0.04)}
& \shortstack{(-8.17)}
& \shortstack{(+0.24)}
& \shortstack{(-8.40)}
& \shortstack{(+0.05)}
& \shortstack{(-11.13)}
& \shortstack{(-0.20)}
& \shortstack{(-1.82)}
& \shortstack{(-1.74)}
& \shortstack{(-0.65)}
& \shortstack{(-2.47)}
\\

\bottomrule
\end{tabular}
\vspace{-0.5cm}
\end{table*}

\begin{table*}[t]
\centering
\caption{
Compression and downstream-task comparison on Qwen1.5-7B.
Metrics, notation, and highlighting follow
Table~\ref{tab:qwen15b_downstream_comparison}.
}
\label{tab:qwen7b_downstream_comparison}

\footnotesize

\setlength{\tabcolsep}{3.0pt}
\renewcommand{\arraystretch}{1.05}

\begin{tabular}{l>{\columncolor{lightgray!20}}c>{\columncolor{lightgray!20}}c >{\columncolor{lightblue!80}}c >
{\columncolor{lightyellow!80}}c>{\columncolor{lightyellow!80}}c>{\columncolor{lightyellow!80}}c>{\columncolor{lightyellow!80}}c>{\columncolor{lightyellow!80}}c>{\columncolor{lightyellow!80}}c>{\columncolor{lightgreen!80}}c>
{\columncolor{lightgreen!80}}c>{\columncolor{lightgreen!80}}c>{\columncolor{lightgreen!80}}c}
\toprule

& \multicolumn{1}{c}{Reference}
& \multicolumn{1}{c}{Dyn-Sys}
& \multicolumn{1}{c}{Previous}
& \multicolumn{6}{c}{Conventional PTQ}
& \multicolumn{4}{c}{Advanced Weight-only}
\\
\cmidrule(lr){2-2}
\cmidrule(lr){3-3}
\cmidrule(lr){4-4}
\cmidrule(lr){5-10}
\cmidrule(lr){11-14}

Metric / Task
& Original
& \texttt{Hilbert}
& IR
& ZQ4
& ZQ8
& SQ4
& SQ8
& PTQ4
& PTQ8
& GPTQ
& AWQ
& SpQR
& QuIP
\\

\hline

\textbf{CR}
& --
& 3.00$\times$
& 2.91$\times$
& 3.21$\times$
& 1.85$\times$
& 3.21$\times$
& 1.85$\times$
& 3.21$\times$
& 1.85$\times$
& 2.63$\times$
& 2.64$\times$
& 2.47$\times$
& 2.69$\times$
\\

\textbf{PPL}
& 11.30
& 11.81
& 12.47
& 13.98
& 11.30
& 13.97
& 11.30
& 14.57
& 11.31
& 11.92
& 11.68
& 11.29
& 11.99
\\

\textbf{($\Delta$ PPL)}
& --
& +0.51
& +1.17
& +2.68
& +0.00
& +2.67
& +0.00
& +3.27
& +0.01
& +0.62
& +0.38
& -0.01
& +0.69
\\

SciQ
& 83.40
& \textbf{85.30}
& 82.70
& 79.50
& 83.80
& 80.20
& 83.80
& 80.80
& 83.90
& 79.00
& 84.80
& 83.10
& 83.90
\\

WinoGrande
& 65.19
& 64.17
& 62.75
& 61.72
& 64.96
& 62.12
& 65.19
& 61.33
& 65.27
& 65.19
& 65.19
& 65.75
& \textbf{66.06}
\\

ARC-Easy
& 63.34
& 61.62
& \textbf{63.68}
& 61.36
& 63.30
& 61.03
& 63.38
& 60.52
& 63.59
& 58.04
& 62.37
& 62.96
& 60.90
\\

ARC-Challenge
& 56.57
& 53.50
& 55.63
& 48.72
& 55.80
& 48.21
& 55.55
& 51.88
& \textbf{56.06}
& 55.12
& 55.46
& 55.97
& 54.69
\\

HellaSwag
& 78.68
& 77.63
& 76.62
& 75.73
& 78.72
& 75.77
& \textbf{78.73}
& 75.16
& 78.64
& 77.66
& 77.74
& 78.36
& 77.76
\\

BoolQ
& 85.47
& 84.80
& 85.08
& 83.70
& 85.32
& 83.88
& \textbf{85.38}
& 84.50
& 85.35
& 84.68
& 85.11
& 85.11
& 85.17
\\

MMLU
& 60.45
& 59.69
& 58.87
& 57.56
& \textbf{60.62}
& 57.62
& 60.58
& 56.24
& 60.53
& 59.72
& 59.97
& 60.13
& 59.68
\\

GSM8K
& 20.92
& 22.14
& 19.26
& 23.43
& 23.35
& 25.17
& 23.05
& 22.59
& 19.48
& \textbf{41.62}
& 23.43
& 18.50
& 25.93
\\

PIQA
& 75.90
& 75.79
& 75.24
& 74.16
& 75.90
& 74.05
& 75.90
& 73.45
& 75.79
& 74.59
& 75.35
& \textbf{76.06}
& 75.46
\\

CB
& 60.71
& 73.21
& 35.71
& \textbf{83.93}
& 60.71
& \textbf{83.93}
& 60.71
& 50.00
& 60.71
& 64.29
& 57.14
& 58.93
& 57.14
\\

COPA
& 86.00
& 86.00
& 86.00
& 78.00
& 86.00
& 81.00
& 84.00
& 81.00
& \textbf{87.00}
& \textbf{87.00}
& 85.00
& 85.00
& 86.00
\\

MultiRC
& 25.68
& 26.88
& 23.10
& \textbf{35.11}
& 24.96
& 34.36
& 25.00
& 31.83
& 25.45
& 28.44
& 29.19
& 25.74
& 23.41
\\

ReCoRD
& 81.06
& 80.66
& 80.44
& 79.57
& 81.21
& 79.90
& 81.20
& 79.40
& 81.11
& 80.20
& 81.07
& \textbf{81.31}
& 80.66
\\

RTE
& 83.03
& 82.67
& 85.20
& 82.67
& 83.39
& 83.39
& 83.39
& 80.51
& 83.75
& 83.75
& 83.39
& 82.31
& \textbf{85.56}
\\

WiC
& 67.55
& 67.08
& 67.08
& 61.44
& 67.08
& 63.64
& 67.08
& 64.11
& 67.40
& 66.46
& 66.61
& \textbf{67.71}
& 65.05
\\

WSC
& 53.85
& \textbf{56.73}
& \textbf{56.73}
& 48.08
& 55.77
& 49.04
& 55.77
& 45.19
& 55.77
& \textbf{56.73}
& 54.81
& 54.81
& 53.85
\\
\textbf{Avg. Score}
& 65.49
& 66.12
& 63.38
& 64.67
& 65.68
& 65.21
& 65.54
& 62.41
& 65.61
& 66.41
& 65.41
& 65.11
& 65.08
\\

\textbf{($\Delta$ Avg)}
& --
& +0.63
& -2.11
& -0.82
& +0.19
& -0.28
& +0.05
& -3.08
& +0.12
& +0.92
& -0.08
& -0.38
& -0.41
\\

\bottomrule
\end{tabular}
\vspace{-0.5cm}
\end{table*}

Table~\ref{tab:qwen15b_downstream_comparison}, Table~\ref{tab:qwen7b_downstream_comparison} and
Figure~\ref{fig:llm_compression_pareto}
show that \texttt{Hilbert} occupies a competitive operating point in the high-compression region across both models.
On Qwen2.5-1.5B, \texttt{Hilbert} achieves a compression ratio of
$2.67\times$ with a PPL increase of only $0.35$.
GPTQ and AWQ reach close matched compression ratios: $2.68\times$ and $2.69\times$, respectively.
However, their PPL increases are $0.98$ and $0.72$, compared with $0.35$ for \texttt{Hilbert}.
Likewise, ZeroQuant-INT4, SQuant-INT4, and PTQ-INT4 reach $2.74\times$ compression but increase PPL by $4.24$, $4.27$,
and $5.88$, respectively.
Compared with IR, \texttt{Hilbert} achieves the nearly identical compression ratio while reducing the PPL increase from $0.58$ to $0.35$.
SpQR provides a lower PPL increase of $0.19$ at a lower compression ratio of $2.50\times$, whereas QuIP reaches a slightly higher compression ratio of $2.75\times$ with a similar PPL increase of $0.37$.
These results place \texttt{Hilbert} at a competitive operating point in the high-compression region of the weight-only comparison.

On Qwen1.5-7B, \texttt{Hilbert} reaches $3.00\times$ compression with a PPL increase of $0.51$.
ZeroQuant-INT4, SQuant-INT4, and PTQ-INT4 achieve somewhat higher compression at $3.21\times$ but incur larger PPL increases of $2.68$, $2.67$ and $3.27$, respectively.
Compared with IR, \texttt{Hilbert} increases the compression ratio from $2.91\times$ to $3.00\times$ while reducing the PPL increase from $1.17$ to $0.51$.
Other methods provide different operating points: AWQ achieves a smaller
PPL increase of $0.38$ at $2.64\times$, while SpQR retains nearly unchanged
perplexity at $2.47\times$ compression.
Relative to these lower-compression operating points, \texttt{Hilbert}
achieves a higher compression ratio of $3.00\times$ at the cost of a
moderate increase in perplexity.

The average downstream results further provide a complementary view of compression quality.
\texttt{Hilbert} obtains average scores of $66.52$ and $66.12$ on Qwen2.5-1.5B and Qwen1.5-7B, corresponding to changes of
$-0.79$ and $+0.63$ points, respectively.
On Qwen1.5-7B, its average score remains higher than those of the three conventional INT4 PTQ baselines despite their somewhat higher compression ratios.
Overall, these results indicate that \texttt{Hilbert} maintains competitive aggregate downstream performance while operating in the high-compression region.

Tables~\ref{tab:qwen15b_downstream_comparison} and ~\ref{tab:qwen7b_downstream_comparison} also provide a more detailed view of compression effects on downstream-task performance. 
For \texttt{Hilbert}, performance remains relatively stable on several broad-coverage benchmarks, including HellaSwag, BoolQ, MMLU, PIQA, and ReCoRD.
On Qwen2.5-1.5B, the score changes on these tasks remain within approximately one percentage point, while on Qwen1.5-7B the same group of tasks exhibits similarly moderate variations.

Meanwhile, the task-wise results reveal substantially larger method-dependent variations on several benchmarks.
GSM8K, CB, MultiRC, and WSC show noticeably larger fluctuations across compression approaches and model scales, with both positive and negative changes observed after compression.
For example, GPTQ obtains a substantial improvement on GSM8K for Qwen1.5-7B, whereas its behaviors on several other tasks do not show the same trend.
Similarly, CB and MultiRC exhibit large variations across multiple quantization families.
These task-dependent fluctuations further show that changes in PPL are not necessarily aligned with changes in individual downstream-task scores.
Consequently, isolated score improvements after compression should not be
interpreted as a consistent gain in model capability.
Overall, \texttt{Hilbert} maintains competitive task-level performance.

\section{Conclusion}
\label{sec:conclusion}

This paper has investigated model compression through the lens of dynamic systems. We first established a theoretical foundation for DS-based compression with the irrational winding, thereby transforming DS-based compression from a heuristic construction into a principled tool with a controllable error profile. We have then extended the paradigm from the single irrational-winding instantiation to a unified framework based on DS. To make the framework practical at scale, we have introduced KD-tree and coordinate-template acceleration for fast state assignment, together with outlier identification and threshold-based fallback for error control. Experiments confirm that DS-based compression is competitive. Several open questions remain for future work, including automatic selection of the dynamic system and its configuration and a more complete characterization of computational overhead during both compression and deployment. We hope this work stimulates further research at the intersection of dynamic systems and efficient deep learning.

\bibliographystyle{IEEEtran}
\bibliography{references}

\end{document}


\renewcommand{\thetable}{S\arabic{table}}
\renewcommand{\thefigure}{S\arabic{figure}}
\renewcommand{\theequation}{S\arabic{equation}}

\title{
Supplementary Materials of ``Dynamics as Code: On Model Compression via Dynamic System''
}

\author{
Fan Gao, Wei Su*, Juntong Fan*, Renfeng Peng, Hongyu Liu, Jinqiao Duan, Feng-Lei Fan
}
\maketitle
\vspace{0.8em}

\begin{IEEEkeywords}
Model compression, dynamic systems, neural networks, parameter representation.
\end{IEEEkeywords}

\section{Detailed Space-Filling curve Constructions}
\label{subsec:mapping_families}
This section provides the detailed index-to-coordinate constructions of the four space-filling mappings used in the main manuscript. All four mappings define deterministic bijections between
a scalar index and a finite \(D\)-dimensional lattice. 
The main manuscript retains their common formulation and compression role,
while the explicit traversal rules are given below for completeness.
\subsection{\texttt{Hilbert} Traversal}
For the \texttt{Hilbert} traversal, let \(m=2^k\). For a scalar index \(\xi\in\{0,\ldots,2^{kD}-1\}\), we define \begin{equation}
    h_0 
    = 
\left\lfloor 
\frac{\xi}{2^{(k-1)D}} 
\right\rfloor, \qquad 
    \xi' = \xi\bmod 2^{(k-1)D}, 
\end{equation}
    where \(h_0\in\{0,\ldots,2^D-1\}\) denotes the first-level subcube index, and \(\xi'\) is the remaining index at the next recursion level. 
The \texttt{Hilbert} lattice state can be written abstractly as
\begin{equation} 
    \Phi_{\mathrm{Hilbert}}^{(D,k)}(\xi)
    =
    2^{k-1}\mathbf{v}_{h_0}
    +
    \mathcal{R}_{h_0}^{(k-1)}
    \left(
    \Phi_{\mathrm{Hilbert}}^{(D,k-1)}(\xi')
    \right),
    \label{eq:hilbert_recursive} 
\end{equation}
    with \(\Phi_{\mathrm{Hilbert}}^{(D,0)}(0)=\mathbf{0}\). 
    Here, \(\mathbf{v}_{h_0}\in\{0,1\}^{D}\) denotes the subcube position, and \(\mathcal{R}_{h_0}^{(k-1)}\) reorients the lower-order Hilbert pattern within the subcube via coordinate permutations and axis reflections to preserve continuity and locality.
    Thus, 
    \begin{equation} 
    \Phi_{\mathrm{Hilbert}}^{(D,k)}(\xi)
    \in
    \{0,\ldots,2^k-1\}^{D}.
    \end{equation}
    
\subsection{\texttt{Peano} Traversal}
    For the \texttt{Peano} traversal, let \(m=3^k\). The scalar index \(\xi\in\{0,\ldots,3^{kD}-1\}\) is defined as 
\begin{equation} 
    p_0 
    = 
    \left\lfloor 
    \frac{\xi}{3^{(k-1)D}} 
    \right\rfloor, 
    \quad 
    \xi' 
    =
    \xi \bmod 3^{(k-1)D}, 
\end{equation}
    where \(p_0\in\{0,\ldots,3^D-1\}\) denotes the first-level subblock index.The corresponding lattice \texttt{Peano} state is written recursively as
    \begin{equation} 
    \Phi_{\mathrm{Peano}}^{(D,k)}(\xi)
    =
    3^{k-1}\mathbf{v}_{p_0}
    +
    \mathcal{R}_{p_0}^{(k-1)}
    \left(
    \Phi_{\mathrm{Peano}}^{(D,k-1)}(\xi')
    \right),
    \label{eq:peano_recursive} 
\end{equation}
    with \(\Phi_{\mathrm{Peano}}^{(D,0)}(0)=\mathbf{0}\). 
    Here, \(\mathbf{v}_{p_0}\in\{0,1,2\}^{D}\) denotes the geometric position of the \(p_0\)-th first-level subblock, and \(\mathcal{R}_{p_0}^{(k-1)}\) denotes the corresponding orientation operation. 
    Thus, 
    \begin{equation} 
    \Phi_{\mathrm{Peano}}^{(D,k)}(\xi)
    \in
    \{0,\ldots,3^k-1\}^{D}.
    \end{equation}

\subsection{\texttt{Morton/Z-order} Traversal}
   \texttt{Morton/Z-order} uses bit interleaving, providing a simpler indexing rule but weaker locality near block boundaries. Let \(m=2^k\) be the grid resolution per-axis, where \(k\) is the number of bits used for each coordinate. For a scalar index \(\xi\in\{0,\ldots,m^D-1\}\), we write its binary representation as
    \begin{equation}
    \xi = \sum_{r=0}^{kD-1} b_r(\xi)2^r, 
    \qquad 
    b_r(\xi)\in\{0,1\}.
    \end{equation}
   Under the \texttt{Morton} bit-interleaving convention, the \(d\)-th coordinate is obtained by collecting every \(D\)-th bit:
    \begin{equation} 
    q_d(\xi) 
    = 
    \sum_{r=0}^{k-1} b_{rD+d}(\xi)2^r, 
    \quad d=0,\ldots,D-1 .
    \label{eq:morton_mapping} 
    \end{equation}
    Accordingly, the state generated by the \texttt{Morton/Z-order} system is
\begin{equation} 
    \Phi_{\mathrm{Morton}}^{(D,k)}(\xi)
    =
    \bigl(
    q_0(\xi),
    \ldots,
    q_{D-1}(\xi)
    \bigr)
    \in
    \{0,\ldots,m-1\}^{D}.
    \label{eq:morton_recursive}
\end{equation}

\subsection{\texttt{Snake} Traversal}
    \texttt{Snake} traversal follows a regular alternating scan pattern and provides a low-complexity grid-ordering rule.
    Let \(m=2^k\) denote the per-axis grid resolution. 
    For a scalar index \(\xi\in\{0,\ldots,m^D-1\}\), we first obtain the index coordinates \(\mathbf{c}(\xi)=\bigl(c_1(\xi),\ldots,c_D(\xi)\bigr)\), where
    \begin{equation}
    \xi
    =
    \sum_{j=1}^{D}c_j(\xi)m^{j-1},
    \qquad
    c_j(\xi)
    =
    \left\lfloor
    \frac{\xi}{m^{j-1}}
    \right\rfloor
    \bmod m.
    \end{equation}
    For \(j=1,\ldots,D-1\), the scan direction along the coordinate \(j\) is determined by the parity of the next higher-order coordinate:
    \begin{equation}
    c_j'(\xi)
    =
    \begin{cases}
    c_j(\xi),
    &
    \displaystyle
    c_{j+1}(\xi)\equiv0\pmod2,
    \\[1pt]
    m-1-c_j(\xi),
    &
    \displaystyle
    c_{j+1}(\xi)\equiv1\pmod2,
    \end{cases}
    \label{eq:snake_recursive}
    \end{equation}
and \(c_D'(\xi)=c_D(\xi)\). Thus,  
    \begin{equation}
    \Phi_{\mathrm{Snake}}^{(D,k)}(\xi)
    =
    \bigl(
    c_1'(\xi),
    \ldots,
    c_{D}'(\xi)
    \bigr)
    \in \{0,\ldots,m-1\}^{D}.
    \label{eq:snake_state}
    \end{equation} 

\begin{figure*}[!t]
\centering
\includegraphics[width=\textwidth]{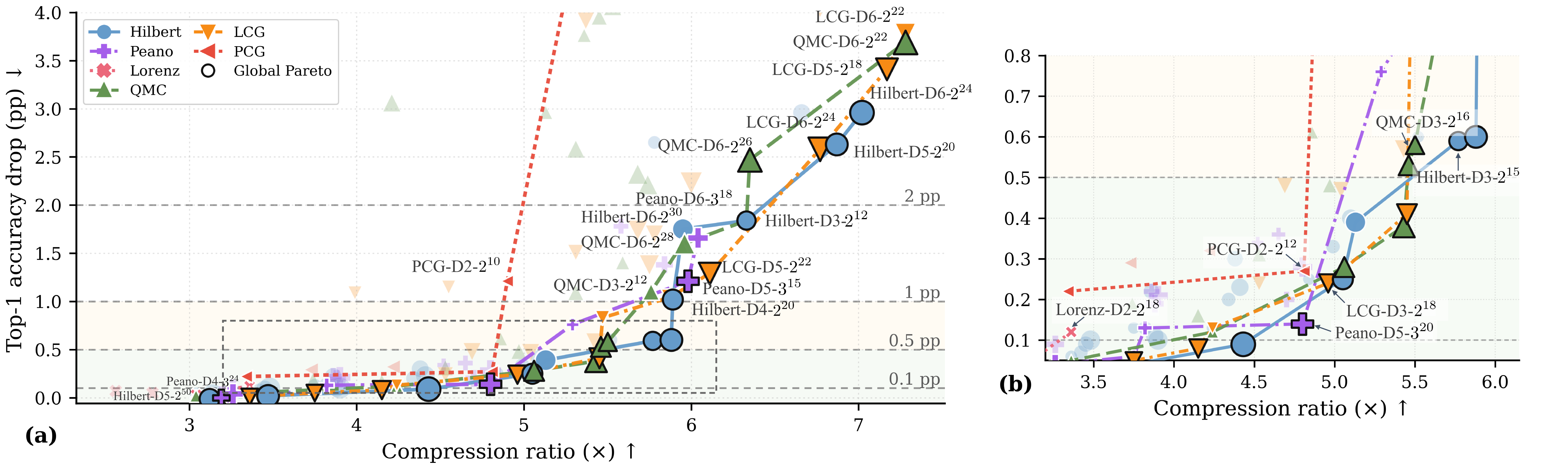}
\caption{
Compression--accuracy method-wise Pareto frontiers on ResNet-18/ImageNet.
The uncompressed model achieves \(69.76\%\) Top-1 accuracy.
For clarity, only configurations with a Top-1 accuracy drop of at most \(4.0\) percentage points are shown in panel (a).
Panel (b) enlarges the high-fidelity region indicated by the dashed rectangle in panel (a).
}
\label{fig:pareto_analysis}
\vspace{-0.5cm}
\end{figure*}

\section{Analysis Experiments}
\label{sec:analysis_experiments}

\begin{figure}[!t]
\centering
\includegraphics[width=\linewidth]{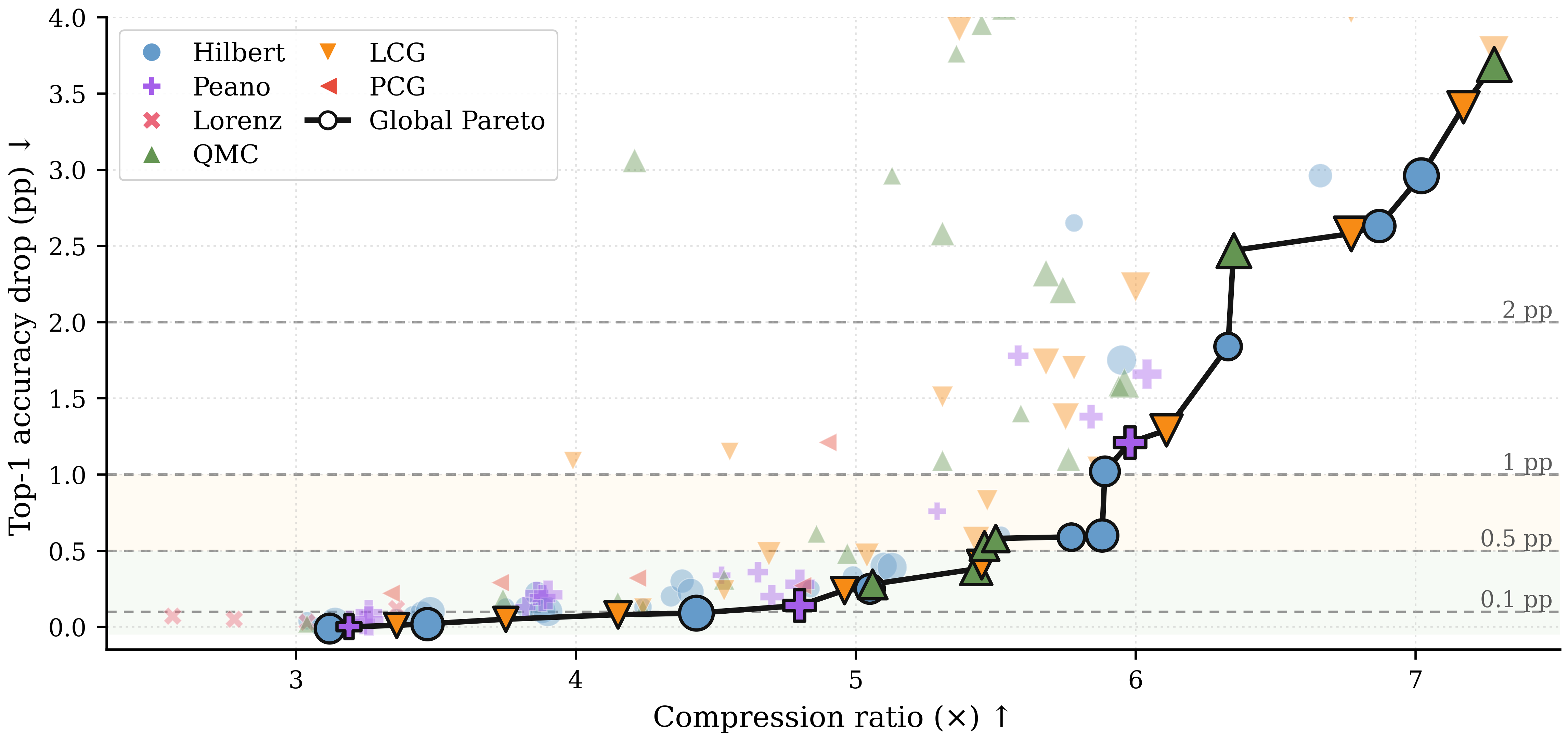}
\caption{
Compression--accuracy global Pareto analysis on ResNet-18/ImageNet.
The uncompressed model achieves \(69.76\%\) Top-1 accuracy.
For clarity, configurations with a Top-1 accuracy drop greater than \(4.0\) percentage points are omitted.
}
\label{fig:pareto_global}
\vspace{-0.5cm}
\end{figure}

In this section, we supplement the global and method-wise Pareto analysis to support the representative configurations.

\textbf{Pareto Analysis}.
\label{subsec:pareto_analysis}
We jointly analyze the performance of the compression ratio and the accuracy drop using global and method-wise Pareto frontiers in Figure~\ref{fig:pareto_global} and Figure~\ref{fig:pareto_analysis}.
A hyper-parameter configuration is Pareto-optimal only if no other configuration achieves an equal or higher compression ratio and an equal or lower accuracy drop.

Figure~\ref{fig:pareto_global} presents the Pareto frontier of all methods with all hyper-parameter configurations.
Because the \texttt{Hilbert}, \texttt{Morton/Z-order}, and \texttt{Snake} curves produce identical results under matched configurations,
their Pareto frontiers coincide. 
We only annotate \texttt{Hilbert} to avoid visual clutter. No single method dominates the Pareto frontier in Figure~\ref{fig:pareto_global}, \textit{i.e.}, no single method is completely superior to the rest. However, in the high-fidelity and moderate-compression regions, the frontier is jointly supported by \texttt{Peano}, \texttt{LCG}, \texttt{Hilbert}, and \texttt{QMC}, and neither \texttt{Lorenz} nor \texttt{PCG} contributes a point to the global frontier. This means that in this region, \texttt{Peano}, \texttt{LCG}, \texttt{Hilbert}, and \texttt{QMC} are better.

Figure~\ref{fig:pareto_analysis} presents the Pareto frontier of each method under different hyper-parameter configurations, which  characterizes the performance of each method.
\texttt{Hilbert} and \texttt{Peano} exhibit relatively stable frontiers but occupy different regions because of their different resolution schedules. \texttt{QMC} and \texttt{LCG} present broader but less regular frontiers, and this frontier spans a large compression range. This non-monotonic character means that the optimal configuration of $D$ and $S$ may appear in different places. \texttt{PCG} provides a limited effective frontier, only \(D=2\) achieves a competitive compression--accuracy trade-off. \texttt{Lorenz} forms a narrow frontier with stable accuracy preservation but limited compression capacity.


As Figure~\ref{fig:pareto_analysis} shows, since a Pareto frontier does not define a unique optimum, we select one representative configuration for each method from its Pareto frontier near the knee region.
Specifically, we prioritize an approximate compression ratio range of \(5\times\) with a drop in Top-1 accuracy below \(0.6\%\).
However, \texttt{Lorenz} does not reach this compression region, we select its highest-compression configuration that maintains near-lossless accuracy. The representative configurations are reported in the main manuscript.
Excluding \texttt{Lorenz}, representative configurations achieve the compression ratio between \(4.80\)--\(5.77\times\), with the Top-1 accuracy drop between \(0.14\)--\(0.59\%\).
Lorenz achieves a \(3.36\times\) compression ratio with a \(0.12\%\) accuracy drop. 

\bibliographystyle{IEEEtran}
\clearpage
\onecolumn

\vfill